%% file: arXiv_SGD-AS_v4.tex
\documentclass{article}

\usepackage{microtype}
\usepackage{graphicx}
\usepackage{booktabs} 
\usepackage{float} 
\usepackage{caption}
\usepackage{subcaption}

\usepackage[accepted]{icml2019blank}

\icmltitlerunning{SGD: General Analysis and Improved Rates}

\input{SGD-AS-header.tex}
\usepackage[colorinlistoftodos,bordercolor=orange,backgroundcolor=orange!20,linecolor=orange,textsize=scriptsize]{todonotes}

\begin{document}

\twocolumn[
\icmltitle{SGD: General Analysis and Improved Rates}


\begin{icmlauthorlist}
\icmlauthor{Robert M.\ Gower}{telecom}
\icmlauthor{Nicolas Loizou}{edi}
\icmlauthor{Xun Qian}{kaust}
\icmlauthor{Alibek Sailanbayev}{kaust}
\icmlauthor{Egor Shulgin}{mipt}
\icmlauthor{Peter Richt\'{a}rik}{kaust,edi,mipt}
\end{icmlauthorlist}

\icmlaffiliation{kaust}{King Abdullah University of Science and Technology, Kingdom of Saudi Arabia}
\icmlaffiliation{mipt}{Moscow Institute of Physics and Technology, Russian Federation}
\icmlaffiliation{edi}{University of Edinburgh, United Kingdom}
\icmlaffiliation{telecom}{T\'{e}l\'{e}com Paris Tech, France}

\icmlcorrespondingauthor{Peter Richt\'{a}rik}{peter.richtarik@kaust.edu.sa}

\icmlkeywords{SGD, arbitrary sampling, importance sampling, minibatches, expected smoothness}

\vskip 0.3in
]

\printAffiliationsAndNotice{}

\begin{abstract}
We propose a  general yet simple theorem describing the convergence of SGD under the {\em arbitrary sampling} paradigm.  
Our theorem describes the convergence of an infinite array of variants of SGD, each of which is associated with a specific probability law governing the data selection rule used to form  minibatches. This is the first time such an analysis is performed, and most of our variants of SGD were never explicitly considered in the literature before.  Our analysis relies on the recently introduced notion of expected smoothness and does not rely on a uniform bound on the variance of the stochastic gradients. By specializing our theorem to different mini-batching strategies, such as sampling with replacement and independent sampling, we derive exact expressions for the stepsize as a function of the mini-batch size. With this we can also determine the mini-batch size that optimizes the total complexity, and show explicitly that as the variance of the stochastic gradient evaluated at the minimum grows, so does the optimal mini-batch size. For zero variance, the optimal mini-batch size is one. Moreover, we prove insightful  stepsize-switching rules  which describe when one should switch from a constant to a decreasing stepsize regime. 
\end{abstract}



%
%

\section{Introduction}

We consider the optimization problem
\begin{equation}
\label{eq:prob}
 x^* = \arg\min_{x\in\R^d} \left[ f(x) = \frac{1}{n} \sum_{i=1}^n f_i(x)\right], 
\end{equation}
where each $f_i: \R^d \to \R$ is smooth (but not necessarily convex). Further, we assume that $f$ has a unique\footnote{This assumption can be relaxed; but for simplicity of exposition we enforce it.} global minimizer $x^*$ and is $\mu$--strongly quasi-convex~\cite{karimi2016linear, Necoara-Nesterov-Glineur-2018-linear-without-strong-convexity}: 
\begin{equation}\label{eq:strconvexcons}
 f(x^*) \geq f(x)+ \dotprod{\nabla f(x) , x^*-x} + \frac{\mu}{2} \norm{x^*-x}^2
\end{equation}
for all $x \in \R^d$.


\subsection{Background and contributions}
Stochastic gradient descent (SGD) \cite{robbins1951stochastic, NemYudin1978, NemYudin1983book, Pegasos, Nemirovski-Juditsky-Lan-Shapiro-2009, HardtRechtSinger-stability_of_SGD}, has become the workhorse for training supervised machine learning problems which have the generic form \eqref{eq:prob}.  

\textbf{Linear convergence of SGD.} 
\citet{moulines2011non} provided a  non-asymptotic analyses of SGD showing  linear convergence for strongly convex $f$ up to a certain noise level. \citet{needell2014stochastic} improved upon these results by removing the quadratic dependency on the condition number in the iteration complexity results, and considered importance sampling.  
The analysis of~\citet{needell2014stochastic} was later extended to a mini-batch variant 
where the mini-batches are formed by partitioning the data~\cite{batchSGDNW16}. These works are the main starting point for ours.  

{\em Contributions:} We further {\em tighten and generalize these results to virtually all forms of sampling.} We introduce an {\em expected smoothness} assumption (Assumption~\ref{ass:Expsmooth}),  first introduced in~\cite{GowerRichBach2018} in the context of a certain class of variance-reduced methods. This assumption is a joint property of $f$ and the sampling scheme $\cD$ utilized by an SGD method, and  allows us prove a generic complexity result (Theorem~\ref{theo:strcnvlin}) that holds for arbitrary sampling schemes $\cD$.  Our work  is the {\em first time SGD is analysed under this assumption.} We obtain {\em linear convergence rates  without strong convexity}; in particular, assuming strong quasi-convexity (this class includes some non-convex functions as well). Furthermore, we {\em do not require the functions $f_i$ to be convex.}

\textbf{Gradient noise assumptions.} \citet{Shamir013} extended the analysis of SGD to convex non-smooth optimization (including the strongly convex case). However, their proofs still rely on the assumption that the variance of the stochastic gradient is bounded for all iterates of the algorithm:
there  exists $c \in \R$ such that $\Exp_{i}\|\nabla f_{i} (x^k)\|^2 \leq c$ for all $k$. The same assumption was used in the analysis of several recent papers \cite{recht2011hogwild, hazan2014beyond, rakhlin2012making}. 
A much more relaxed weak growth assumption $\Exp_{i}\|\nabla f_{i} (x^k)\|^2 \leq c_1 + c_2 \Exp\|\nabla f(x^k)\|^2$ for all $k$,  was apparently first used in the later 90's to prove the asymptotic convergence of SGD (see Proposition 4.2 of~\citet{Bertsekas:neurodynamic}). 
 \citet{bottou2018optimization} establish a linear convergence of SGD under this weak growth assumption. Recently, \citet{pmlr-v80-nguyen18c} turn this assumption into a theorem by establishing formulas $c_1$ and $c_2$ under some reasonable conditions, and  provide further insights into the workings of SGD and its parallel asynchronous cousin, Hogwild!. Similar conditions have been also proved and used in the analysis of decentralized variants of SGD \cite{lian2017can,assran2018stochastic}. 
 Based on a strong growth condition ($c_1 =0$), \citet{schmidt2013fast} were the first to establish linear convergence of SGD, with~\citet{Volkan_Bang:2017} later giving sufficient and necessary conditions for the linear convergence of SGD under this condition. 
%


{\em Contributions:} Our analysis does not directly assume a growth condition. Instead, we make use of the remarkably weak {\em expected smoothness} assumption.

\textbf{Optimal mini-batch size.} Recently it was experimentally shown by \citet{imagenet1hour} that using larger mini-batches sizes is key to efficient training of large scale non-convex problems, leading to the training of ImageNet in under 1 hour. The authors conjectured that the stepsize should grow linearly with the mini-batch size. 

{\em Contributions:} We prove (see Section~\ref{sec:specialexample}) that this is the case, upto a certain optimal mini-batch size, and provide {\em exact formulas for the dependency of the stepsizes on the mini-batch sizes. }

\textbf{Learning schedules.} \citet{pmlr-v84-chee18a}  develop techniques for detecting the convergence of SGD within a region around the solution. 

{\em Contributions:} We provide a closed-form formula for when  should SGD {\em switch from a constant stepsize to a decreasing stepsize} (see Theorem~\ref{theo:decreasingstep}). Further, we clearly show how the {\em optimal stepsize (learning rate) increases and the iteration complexity decreases as the mini-batch size} increases for both independent sampling and sampling with replacement. We also recover the well known $L/\mu \log (1/ \epsilon)$ convergence rate of gradient descent (GD) when the mini-batch size is $n$; this is the {\em first time a generic SGD analysis recovers the correct rate of GD.}

\textbf{Over-parameterized models.} There has been some recent work in analysing SGD in the 
setting where the underlying model being trained has more parameters than there is data available. 
In this \emph{zero--noise} setting, \citet{MaBB18} showed that SGD converges linearly.

{\em Contributions:}  In the case of over-parametrized models, we extend the findings of~\citet{MaBB18}\footnote{Recently, the results of \citet{MaBB18} were extended to the accelerated case by~\citet{vaswani2018fast}; however, we do not study accelerated methods in this work.} to independent sampling and sampling with replacement by showing that the optimal mini-batch size is $1$.   Moreover, we provide results in the more general setting where the model is not necessarily over-parametrized.

 \textbf{Practical performance.} We  corroborate our theoretical results with extensive experimental testing.
  



\subsection{Stochastic reformulation}

In this work we provide a single theorem through which we can analyse all importance sampling and mini-batch variants of SGD. 
To do this, we need to introduce a {\em sampling vector} which we will use to re-write our problem~\eqref{eq:prob}.

 

\begin{definition}\label{def:unbiasedv} We say that a random vector $v \in \R^n$ drawn from some distribution $\cD$ is a {\em sampling vector}  if its mean is the vector of all ones:
\begin{equation}\label{eq:nuh98f}
\EE{\cD}{v_i} =1, \quad \forall i\in [n].
\end{equation}
\end{definition}
With each distribution $\cD$ we now introduce a \emph{stochastic reformulation} of~\eqref{eq:prob} as follows
\begin{equation} \label{eq:reformulation}
 \min_{x\in \R^d} \; \EE{\cD}{f_v(x) \eqdef \frac{1}{n}\sum_{i=1}^n v_i f_i(x)}.
 \end{equation}
By the definition of the sampling vector, $f_v(x)$ and  $\nabla f_v(x)$ are unbiased estimators of $f(x)$ and $\nabla f(x),$ respectively, and hence probem \eqref{eq:reformulation} is indeed equivalent (i.e., a reformulation) of the original problem \eqref{eq:prob}. In the case of the gradient, for instance, we get  \begin{equation}\label{eq:unbiasedgrad}
 \EE{\cD}{\nabla f_v(x)} \overset{\eqref{eq:reformulation}}{=} \frac{1}{n}\sum_{i=1}^n \EE{\cD}{v_i} \nabla f_i(x) \overset{\eqref{eq:nuh98f}}{=} \nabla f(x).
\end{equation}

Similar but different stochastic reformulations were recently proposed by \citet{ASDA} and further used in \citep{loizou2017momentum,loizou2019convergence} for the more special problem of solving linear systems, and by \citet{GowerRichBach2018} in the context of variance-reduced methods.  Reformulation~\eqref{eq:reformulation} can be  solved using  SGD in a natural way:
 \begin{eqnarray} 
\boxed{ x^{k+1} =  x^k- \gamma^k \nabla f_{v^k}(x^k)}  \label{eq:sgdstep}
\end{eqnarray}  
where $v^k \sim \cD$ is sampled i.i.d.\ at each iteration and $\gamma^k>0$ is a  stepsize. However, for different distributions $\cD$, \eqref{eq:sgdstep} has a {\em different interpretation} as an SGD method for solving the {\em original} problem \eqref{eq:prob}. In our main result we will analyse \eqref{eq:sgdstep} for any $\cD$ satisfying \eqref{eq:nuh98f}. By substituting specific choices of $\cD$, we obtain specific variants of SGD for solving \eqref{eq:prob}.

\section{Expected Smoothness and Gradient Noise}

In our analysis of SGD \eqref{eq:sgdstep} applied to the stochastic reformulation \eqref{eq:reformulation} we rely on a generic and remarkably weak assumption of {\em expected smoothness}, which we now define  and relate to existing growth conditions.

\subsection{Expected smoothness}

Expected smoothness~\cite{GowerRichBach2018} is an assumption that combines both the properties of the distribution $\cD$ and the smoothness properties of  function $f$.
\begin{assumption}[Expected Smoothness]
\label{ass:Expsmooth} We say that $f$ is $\cL$--smooth in expectation with respect to distribution $\cD$ if there exists  $\cL=\cL(f,\cD)>0$  such that
\begin{equation}
\label{eq:expsmooth}
\EE{\cD}{\norm{\nabla f_v(x)-\nabla f_v(x^*)}^2} \leq 2\cL (f(x)-f(x^*)),
\end{equation}
for all $x\in\R^d$. For simplicity, we will write $(f,\cD)\sim ES(\cL)$ to say that \eqref{eq:expsmooth} holds.  When $\cD$ is clear from the context, we will often ignore mentioning it,  and simply state that the expected smoothness constant is $\cL.$
\end{assumption}

There are scenarios where the above inequality is tight. Indeed, in the setting of stochastic reformulations of linear systems considered in \cite{ASDA}, one has $f_v(x) = \frac{1}{2}\|\nabla f_v(x)\|^2$, $\nabla f_v(x^*) = 0$ and $f_v(x^*)=0$, which means that \eqref{eq:expsmooth} holds as an identity with $\cL=1$.

In Section~\ref{sec:cLandsigma} we  show how convexity and $L_i$--smoothness of $f_i$ implies expected smoothness. However, the opposite implication does not hold. Indeed, the expected smoothness assumption can hold even when the $f_i$'s and $f$ are not convex, as we show in the next example.

\begin{example}[Non-convexity and expected smoothness] Let $f_i=\phi$ for $i=1,\ldots, n$, where $\phi$ is a $L_{\phi}$--smooth and non-convex function which has a global minimum $x^*\in\R^d$ (such functions exist\footnote{\tiny There exists invex functions that satisfy these conditions~\cite{karimi2016linear}. As an example $\phi(x) = x^2 + 3 \sin^2(x)$ is smooth, non-convex, and has a unique global minimizer.}). Consequently $f = \phi$  and $f_v = \frac{\sum_i v_i}{n} \phi $. Letting 
$\theta \eqdef  \mathbb{E}_{\cD}\big[\big(\sum_i v_i \big)^2\big],$ we have
\begin{align*}
\EE{\cD}{\|\nabla f_v(x) - \nabla f_v(x^*)\|^2} &= \frac{\theta}{n^2}  \|\nabla \phi(x) - \nabla \phi(x^*)\|^2 \nonumber \\
&\leq  \frac{2 \theta L_\phi}{n^2} (f(x)-f(x^*)),
\end{align*}
 where the last inequality follows from  Proposition~\ref{prop:ihs9hd}.
So, $(f,\cD)\sim ES(\cL)$ for $\cL = \frac{\theta L_\phi}{n^2} $. 
\end{example}

%

\subsection{Gradient noise}

Our second key assumption is finiteness of gradient noise, defined next: 
\begin{assumption}[Finite Gradient Noise] \label{ass:grad-noise} 
 The {\em gradient noise} $\sigma=\sigma(f,\cD)$, defined by
\begin{equation}\label{eq:sigma_def} \sigma^2  \eqdef \Exp_{\cD}[\norm{\nabla f_v(x^*)}^2],\end{equation}
is finite.
\end{assumption}

This is a very weak assumption, and should intuitively be really seen as an assumption on $\cD$ rather than on $f$. For instance, if the sampling vector $v$ is non-negative with probability one and $\Exp[v_i \sum_j v_j]$ is finite for all $i$, then $\sigma$ is finite. When \eqref{eq:prob} is the training problem of an over-parametrized model, which  often occurs in deep neural networks, each individual loss function $f_i$ attains its minimum at $x^*$, and 
thus $\nabla f_i(x^*) =0.$  It follows that $\sigma=0$.

\subsection{Key lemma and connection to the weak growth condition} \label{sec:bifg79f9d}

A common assumption used to prove the convergence of SGD is uniform boundedness of the stochastic gradients\footnote{Or it is assumed that $\Exp\|\nabla f_{v} (x^k)\|^2 \leq c$ for all $k$ iterates. But this too has issues since it implicitly assumes that the iterates remain within a compact set, and yet it it used to prove the convergence to within a compact set, raising issues of  a circular argument.}: there exist $0<c < \infty$ such that $\Exp\|\nabla f_{v} (x)\|^2 \leq c$ for all $x$. However, this assumption often does not hold, such as in the case when $f$ is strongly convex~\cite{bottou2018optimization, pmlr-v80-nguyen18c}. 
We do not assume such a bound. Instead, we use the following direct consequence of expected smoothness to bound the expected norm of the stochastic gradients.
%
\begin{lemma}
\label{lem:weakgrowth}
If $(f,\cD)\sim ES(\cL)$, then
\begin{align}
\label{upperbound}
\Exp_{\cD} \left[ \|\nabla f_{v} (x)\|^2 \right] & \leq  4  \cL ( f(x)-f(x^*) ) + 2 \sigma^2.
\end{align}
\end{lemma}

When the gradient noise is zero ($\sigma =0$),  inequality~\eqref{upperbound} is known as the {\em weak growth condition}~\cite{vaswani2018fast}. We have the following corollary:
\begin{corollary}\label{cor:cLoverimpweak}
If $(f,\cD)\sim ES(\cL)$ and if $\sigma =0$, then $f$ satisfies the weak growth condition 
\[\Exp_{\cD}[\norm{\nabla f_v(x)}^2] \leq 2\rho  ( f(x)-f(x^*) ),\]
with $\rho = 2\cL.$
\end{corollary}

This corollary should be contrasted with Proposition~2 in~\cite{vaswani2018fast} and Lemma~1 in~\cite{pmlr-v80-nguyen18c}, where it is shown, by assuming the $f_i$ functions to be smooth and convex, that the weak growth condition holds with $\rho = 2L_{\max}$. However, as we will show in Lemma~\ref{lem:98s9g8s},  $L_{\max} \geq \cL$, and hence our bound is often tighter.

\section{Convergence Analysis}

\subsection{Main results}

We now present our main theorem, and include its proof to highlight how we make use of expected smoothness and gradient noise.
\begin{theorem}\label{theo:strcnvlin}
Assume $f$ is $\mu$-quasi-strongly convex and that $(f,\cD)\sim ES(\cL)$.
Choose   $\gamma^k=\gamma \in (0,  \frac{1}{2\cL}]$ for all $k$. Then iterates of  SGD  given by (\ref{eq:sgdstep}) satisfy: 
\begin{equation}\label{eq:convsgd}
\mathbb{E} \| x^k - x^* \|^2 \leq \left( 1 - \gamma \mu \right)^k \| x^0 - x^* \|^2 + \frac{2 \gamma \sigma^2}{\mu}. 
\end{equation}
Hence, given any $\epsilon>0$, choosing  stepsize
\begin{equation}\label{eq:stepbndmax}
\gamma  = \min \left\{ \frac{1}{2\cL},\; \frac{\epsilon\mu}{4 \sigma^2}\right\},
\end{equation}
and 
\begin{equation}\label{eq:itercomplexlin}
k\geq  \max \left\{ \frac{2\cL}{\mu },\; \frac{4 \sigma^2}{\epsilon\mu^2}\right\} \log\left(\frac{ 2 \|x^0 - x^*\|^2 }{  \epsilon }\right),
\end{equation}
implies $\mathbb{E} \| x^k - x^* \|^2  \leq \epsilon.$
\end{theorem}

\begin{proof}
Let $r^k = x^k -x^*$. From (\ref{eq:sgdstep}), we have 
\begin{align*}
\label{najxs}
\| r^{k+1}  \|^2 &\overset{\eqref{eq:sgdstep}}{ =} \; \|  x^k -x^* -\gamma \nabla f_{v^k}(x^k) \|^2\notag\\
&=\; \|  r^k \|^2 - 2\gamma \langle  r^k, \nabla f_{v^k}(x^k) \rangle + \gamma^2 \|\nabla f_{v^k} (x^k)\|^2. \notag
\end{align*}
Taking expectation conditioned on $x^k$ we obtain:
\begin{align*}
\Exp_{\cD}{\|r^{k+1}\|^2}
\overset{\eqref{eq:unbiasedgrad}}{ = } & \; \| r^k \|^2 - 2\gamma \langle r^k, \nabla f(x^k) \rangle \\
&+ \gamma^2 \Exp_{\cD}\|\nabla f_{v^k} (x^k)\|^2\\
\overset{\eqref{eq:strconvexcons}}{ \leq } & \; (1- \gamma \mu) \| r^k \|^2 - 2\gamma [f(x^k)-f(x^*)]  \\
& \;\; + \gamma^2 \Exp_{\cD}\|\nabla f_{v^k} (x^k)\|^2.
\end{align*}
Taking expectations again and using Lemma \ref{lem:weakgrowth}:
\begin{align*}
\Exp{\|r^{k+1}\|^2}
\overset{\eqref{upperbound}}{ \leq } & \; (1- \gamma \mu) \Exp \| r^k \|^2 + 2 \gamma^2 \sigma^2 \\
& \;\;  + 2\gamma (2\gamma \cL- 1)  \Exp [f(x^k)-f(x^*)] \\
\leq &\; (1-\gamma \mu) \Exp \| r^k \|^2 + 2\gamma^2\sigma^2,
\end{align*}
where we used in the last inequality that $2\gamma \cL \leq  1$ since $\gamma \leq \frac{1}{2\cL}.$
Recursively applying the above and summing up the resulting geometric series gives
\begin{eqnarray}
\mathbb{E} \|r^k\|^2 &\leq &  \left( 1 - \gamma \mu \right)^k \|r^0\|^2 + 2\sum_{j=0}^{k-1} \left( 1 - \gamma \mu \right)^j \gamma^2\sigma^2 \nonumber\\
&\leq &   \left( 1 - \gamma \mu \right)^k \|r^0\|^2 + \frac{2\gamma \sigma^2}{\mu}.\label{eq:la9ja38jf}
\end{eqnarray}
To obtain an iteration complexity result from the above, we use standard techniques as shown in Section~\ref{sec:itercomplextheo}.
\end{proof}

Note that we do not assume $f_i$ nor  $f$ to be convex. Theorem~\ref{theo:strcnvlin} states that SGD converges linearly up to the additive constant $ 2 \gamma \sigma^2/\mu$  which depends on the gradient noise $\sigma^2$ and on the stepsize $\gamma$. We obtain a more accurate solution with a smaller stepsize, but then the convergence rate slows down.  Since we control $\cD$, we  also control $\sigma^2$ and $\cL$ (we compute these parameters for several distributions $\cD$ in Section~\ref{sec:cLandsigma}).

Furthermore, we can control this additive constant by carefully choosing the stepsize, as shown in the next result.

\begin{theorem}[Decreasing stepsizes]
\label{theo:decreasingstep}
Assume $f$ is $\mu$-quasi-strongly convex and that $(f,\cD)\sim ES(\cL)$. Let   $\mathcal{K} \eqdef \left.\cL\right/\mu$ and 
\begin{equation}\label{eq:gammakdef}
\gamma^k= 
\begin{cases}
\displaystyle \frac{1}{2\cL} & \mbox{for}\quad k \leq 4\lceil\mathcal{K} \rceil \\[0.3cm]
\displaystyle \frac{2k+1}{(k+1)^2 \mu} &  \mbox{for}\quad k > 4\lceil\mathcal{K} \rceil.
\end{cases}
\end{equation}
If $k \geq 4 \lceil\mathcal{K} \rceil$, then SGD iterates given by (\ref{eq:sgdstep}) satisfy:
\begin{equation}\label{eq:rateofdecreasing}
\mathbb{E}\| x^{k} - x^*\|^2 \le   \frac{\sigma^2 }{\mu^2 }\frac{8 }{k} + \frac{16 \lceil\mathcal{K} \rceil^2}{e^2 k^2 }  \|x^0 - x^*\|^2.\end{equation}
\end{theorem}


\subsection{Choosing $\cD$} \label{sec:samplingD}
For~\eqref{eq:sgdstep} to be efficient, the sampling  vector $v$ should be sparse. For this reason we will construct $v$ so that only a (small and random) subset of its entries  are non-zero. 

Before we formally define $v$, let us first establish some random set terminology. Let $C\subseteq [n]$ and let $e_C \eqdef \sum_{i\in C} e_i$, where $\{e_1,\dots,e_n\}$ are the standard basis vectors in $\R^n$.  These subsets will be selected using a random set valued map $S$, in the literature referred to by the name {\em sampling}~\cite{PCDM, ESO}.
A sampling is uniquely characterized by choosing subset probabilities $p_C\geq 0$ for all subsets $C$ of $[n]$:
\begin{equation}
\Prob{S = C} = p_C,\quad \forall C \subset [n],
\end{equation}
where $\sum_{C\subseteq [n]} p_C =1$. We will only consider \emph{proper} samplings.
A sampling $S$ is called proper if  $p_i \overset{\rm def}{=} \mathbb{P}[i\in S] = \sum_{C:i\in C}p_C$ is positive for all $i$. 

The first analysis of a randomized optimization method with an {\em arbitrary (proper) sampling} was performed by \citet{NSync} in the context of randomized coordinate descent for strongly convex functions. This arbitrary sampling paradigm was later adopted in many other settings, including accelerated coordinate descent for strongly convex functions \cite{FilipACD18}, coordinate and accelerated descent for convex functions \cite{ALPHA}, primal-dual methods \cite{Quartz, SCP}, variance-reduced methods with convex \cite{dfSDCA} and nonconvex \cite{nonconvex_arbitrary} objectives. Arbitrary sampling arises as a special case of our more general analysis by specializing the sampling vector to one dependent on a sampling $S$. We now define practical sampling vector $v=v(S)$ as follows:
\begin{lemma}
Let $S$ be a proper sampling, and let ${\hat \mP} = {\rm Diag}(p_1, ..., p_n).$
Then the random vector $v=v(S)$ given by
\begin{equation}\label{eq:asvector}
v = {\hat \mP}^{-1} e_S
\end{equation}
is a sampling vector.
\end{lemma}
\begin{proof}
Note that $v_i =  \mathbf{1}_{(i\in S)}/p_i,$ where $\mathbf{1}_{(i\in S)}$ is the indicator function of the event $i\in S$. It follows that 
$\E{v_i} = \E{\mathbf{1}_{(i\in S)}}/p_i = 1$.
\end{proof}

We can further specialize and define the following commonly used samplings. Each sampling $S$ gives rise to a particular sampling vector $v=v(S)$ (i.e., distribution $\cD$), which in turn gives rise to a particular stochastic reformulation~\eqref{eq:reformulation} and SGD variant~\eqref{eq:sgdstep}.

\textbf{Independent sampling.}
The sampling $S$ includes every $i$, independently, with probability $p_i>0$. This type of sampling was considered in different contexts in~\cite{nonconvex_arbitrary, FilipACD18}. 

\textbf{Partition sampling.}
A partition ${\cal G}$ of $[n]$ is a set consisting of subsets of $[n]$ such that $\cup_{C\in {\cal G}} C = [n]$ and $C_i \cap C_j = \emptyset$ for any $C_i$, $C_j\in {\cal G}$ with $i\neq j$. A partition sampling $S$ is a sampling such that $p_C = \mathbb{P}[S=C] > 0$ for all $C \in {\cal G}$ and $\sum_{C \in {\cal G}} p_C = 1$. 

\textbf{Single element sampling.}
Only the singleton sets $\{i\}$ for $i=1,\ldots, n$ have a non-zero probability of being sampled; that is,
$\Prob{|S|=1} = 1$. We have $\Prob{v(S) = e_i/p_i} = p_i$.

\textbf{$\tau$--nice sampling.}
We say that $S$ is a $\tau$--nice if $S$ samples from all subsets of $[n]$ of cardinality $\tau$ uniformly at random. In this case we have that $p_i = \frac{\tau}{n}$ for all $i \in [n].$ So, $\Prob{v(S) =\frac{n}{\tau}e_C} = \left. 1 \right/ \binom{n}{\tau}$ for all subsets $C \subseteq \{1,\ldots, n\}$ with $\tau$ elements. 
  
%

\subsection{Bounding $\cL$ and $\sigma^2$}
\label{sec:cLandsigma}

By assuming that the $f_i$ functions are convex and smooth 
we can calculate closed form expressions for the expected smoothness $\cL$ and gradient noise $\sigma^2$.   
In particular we make the following smoothness assumption:

\begin{assumption}\label{ass:Mismooth}
There exists a symmetric positive definite matrix $\mM_i\in \R^{d\times d}$   such that
\begin{equation}\label{eq:Mismooth}
f_i(x+h) \leq f_i(x) + \dotprod{\nabla f_i(x), h} + \frac{1}{2}\norm{h}_{\mM_i}^2,
\end{equation}
for all $x,h\in \R^d,$ and $i \in [n],$ where $\norm{h}_{\mM_i}^2 \eqdef \dotprod{\mM_i h, h}.$ In this case we say that $f_i$ is $\mM_i$--smooth. Furthermore, we assume that each $f_i$ is convex.
\end{assumption}

To better relate the above assumption to the standard smoothness assumptions we make the following remark.
\begin{remark} \label{rem:LmaxLi}
As a consequence of Assumption~\ref{ass:Mismooth} we also have that each $f_i$ is $L_i \eqdef \lambda_{\max}(\mM_i)$--smooth and $f$ is  $L \eqdef \frac{1}{n}\lambda_{\max}(\sum_{i=1}^n\mM_i)$--smooth. Let $L_{\max} \eqdef \max_{i \in [n]}L_i.$
\end{remark}

Using Assumption~\ref{ass:Mismooth} and a sampling  we establish the following bounds on $\cL$.
\begin{theorem}\label{theo:expsmoothas}
	Let $S$ be a proper sampling, and  $v=v(S)$ (i.e., $v$ is defined by (\ref{eq:asvector}). Let $f_i$ be $\mM_i$-smooth,  and $\mP \in \mathbb{R}^{n\times n}$ be defined by $\mP_{ij} = \mathbb{P}[i\in S \ \& \ j\in S]$. Then $(f,\cD)\sim ES(\cL)$, where
	\begin{eqnarray}\label{eq:expsmoothas}
	{\cal L} &\leq&  \cL_{\max} :=\max_{i\in [n]} \left\{ \sum_{C: i\in C}\frac{p_C}{p_i}L_C \right\} \nonumber \\
	&\leq & \frac{1}{n} \max_{i\in [n]} \left\{ \sum_{j \in [n]} \mP_{ij}\frac{\lambda_{\max}(\mM_j)}{p_ip_j} \right\},
	\end{eqnarray}
	and $L_C \eqdef \frac{1}{n}\lambda_{\max}(\sum_{j \in C} \frac{1}{p_j}\mM_j)$. If $|S| \equiv \tau$, then
	\begin{equation}\label{eq:cLmaxbound}
L	\; \leq   \; \cL_{\max} \; \leq \; L_{\max} = \max_{i\in [n]} \lambda_{\max}(\mM_i).
	\end{equation}
\end{theorem}
By applying the above result to specific samplings, we obtain the following practical bounds on $\cL$:
 
\begin{proposition}\label{pro:lmaxi} 
	(i) For single element sampling $S$, we have 
	\begin{equation}\label{eq:lmaxi1}
		\cL_{\max} = \frac{1}{n}\max_{i\in [n]} \frac{\lambda_{\max}(\mM_i)}{p_i}. 
	\end{equation}
	(ii) For partition sampling $S$ with partition ${\cal G}$, we have 
	\begin{equation}\label{eq:lmaxips}
	 \cL_{\max} = \frac{1}{n} \max_{C\in {\cal G}} \left\{ \frac{1}{p_C} \lambda_{\max}(\sum_{j\in C}\mM_j) \right\}.
	\end{equation}
\end{proposition}

For $\tau$-nice sampling and independent sampling, we get the following very informative bounds on ${\cal L}$. 
\begin{proposition}\label{pro:cL} 
(iii) For independent sampling $S$, we have 
\begin{equation}\label{eq:cLis}
 \cL \leq   L + \max_{i\in [n]} \frac{1-p_i}{p_i} \frac{\lambda_{\max}(\mM_i)}{n}. 
\end{equation}

(iv) For $\tau$-nice sampling, we have 

\begin{equation}\label{eq:cLns}
 \cL \leq \frac{n(\tau-1)}{\tau(n-1)}L + \frac{n-\tau}{\tau(n-1)}\max_{i}\lambda_{\max}({\bf M}_i)
\end{equation}

\end{proposition}

\citet{SAGAminib} were the first to suggest using~\eqref{eq:cLns} as an approximation for $\cL$.  Through extensive experiments, they showed that the bound~\eqref{eq:cLns} is very tight. Here we give the first proof that~\eqref{eq:cLns} is indeed a valid upper bound.

For $v=v(S)$ given by~\eqref{eq:asvector}, formulas for the gradient noise $\sigma^2$ are provided in the next result:
\begin{theorem}\label{lem:sigma}
	Let $h_i = \nabla f_i(x^*)$. Then 
	\begin{equation}\label{eq:sigma}
	 \sigma^2 =  \frac{1}{n^2} \sum_{i,j\in [n]} \frac{\mP_{ij}}{p_ip_j} \langle h_i, h_j \rangle.
	\end{equation}
\end{theorem}
Specializing the above theorem to specific samplings $S$ gives the following formulas for $\sigma^2$:
\begin{proposition}\label{pro:sigma}	
	(i) For single element sampling $S$, we have 
	\begin{equation}\label{eq:sigma1}
	 \sigma^2 = \frac{1}{n^2} \sum_{i\in [n]} \frac{1}{p_i} \|h_i\|^2. 
	\end{equation}

	(ii) For independent sampling $S$ with $\mathbb{E}[|S|] = \tau$, we have 
	\begin{equation}\label{eq:sigmais}
	 \sigma^2 = \frac{1}{n^2} \sum_{i\in [n]} \frac{1-p_i}{p_i} \|h_i\|^2.
	\end{equation}
	
	(iii) For $\tau$-nice sampling $S$, we have 
	\begin{equation}\label{eq:sigmans}
	 \sigma^2 = \frac{1}{n\tau} \cdot \frac{n-\tau}{n-1} \sum_{i\in [n]} \|h_i\|^2. 
	\end{equation}
	
	(iv) For partition sampling $S$ with partition ${\cal G} $, we have 
	\begin{equation}\label{eq:sigmaps}
	 \sigma^2 = \frac{1}{n^2} \sum_{C\in {\cal G}} \frac{1}{p_C}\| \sum_{i\in C} h_i\|^2. 
	\end{equation}
\end{proposition}

Generally, we do not know the values of $h_i = \nabla f_i(x^*)$. But if we have prior knowledge that $x^*$ belongs to some set ${\cal C}$, we can obtain upper bounds for $\sigma^2$ for these samplings from Proposition~\ref{pro:sigma} in a straightforward way.


\section{Optimal Mini-Batch Size}
\label{sec:specialexample}

Here we develop the iteration complexity for different samplings by plugging in the bounds on $\cL$ and $\sigma$ given in Section~\ref{sec:cLandsigma} into Theorem~\ref{theo:strcnvlin}. To keep the notation brief, in this section we drop the logarithmic term $ \log\left(2 \|x^0 - x^*\|^2 /  \epsilon \right)$ from the iteration complexity results.
Furthermore, for brevity and to better compare our results to others in the literature,  we will use $L_i = \lambda_{\max}(\mM_i)$  and $L_{\max} = \max_{i\in [n]}L_i$  (see Remark~\ref{rem:LmaxLi}). Finally let $\overline{h} =  \frac{1}{n} \sum_{i\in [n]} \| h_i\|^2$ for brevity.

\textbf{Gradient descent.} As a first sanity check, we consider the case where $|S|= n$ with probability one. That is, each iteration~\eqref{eq:sgdstep} uses the full batch gradient. Thus $\sigma =0$ and it is not hard to see that for $\tau =n$ in~\eqref{eq:cLns} or $p_i=1$ for all $i$ in~\eqref{eq:cLis} we have $\cL_{\max} =L.$ Consequently, the resulting iteration complexity~\eqref{eq:itercomplexlin} is now $k \geq 2L/\mu$. This is exactly the rate of gradient descent, which is precisely what we would expect since the resulting method {\em is} gradient descent. Though an obvious sanity check, we believe this is the first convergence theorem of SGD that includes gradient descent as a special case. Clearly, this is a necessary pre-requisite if we are to hope to understand the complexity of mini-batching.

\subsection{Nonzero gradient noise}

To better appreciate how our iteration complexity evolves with increased mini-batch sizes, we now consider independent sampling with $|S| = \tau$ and $\tau$-nice sampling.

\textbf{Independent sampling.}
 Inserting the bound on $\cL$~\eqref{eq:cLis} and $\sigma$~\eqref{eq:sigmais} into~\eqref{eq:itercomplexlin} gives the following iteration complexity
\begin{align}\label{eq:itercomplexlinis}
k & \geq    \frac{2}{\mu }\max \left\{  L + \max_{i\in [n]} \frac{1-p_i}{np_i}L_i \right., \left.  \frac{2}{\mu \epsilon}\frac{1-p_i}{np_i} \overline{h}\right\}.
\end{align} 
This is a completely new mini-batch complexity result, which opens up the possibility of optimizing the mini-batch size and probabilities of sampling.
For instance, if we fix uniform probabilities with $p_i = \frac{\tau}{n}$ then~\eqref{eq:itercomplexlinis} becomes $k \geq   \frac{2}{\mu }\max \left\{  l(\tau), r(\tau) \right\}$, where
\begin{equation}
l(\tau)   \eqdef L +  \left(\frac{1}{\tau}-\frac{1}{n}\right)L_{\max}; \; 
r(\tau) \eqdef  \frac{2}{\mu \epsilon} \left(\frac{1}{\tau}-\frac{1}{n}\right)\overline{h}. \label{eq:lrforis}
\end{equation}
This complexity result corresponds to using the stepsize
\begin{align}\label{eq:gammalinisuni}
\gamma &    =   \frac{1}{2}\min \left\{  \frac{1}{  l(\tau)}, \frac{1}{r(\tau)} \right\}
\end{align} 
if $\tau <n$, otherwise only the left-hand-side term in the minimization remains. The stepsize~\eqref{eq:gammalinisuni} is increasing since both $l(\tau)$ and $r(\tau)$ decrease as $\tau$ increases.

With such a simple expression for the iteration complexity we can choose a mini-batch size that optimizes the \emph{total complexity}. By defining the \emph{total complexity} $T(\tau)$ as the number of iterations $k$ times the number of gradient evaluations ($\tau$) per iteration  gives 
\begin{eqnarray}\label{eq:totcomplexlinisuni}
  T(\tau) \eqdef \frac{2}{\mu n }\max \left\{  \tau nL +  \left(n - \tau\right)L_{\max} ,\frac{2\left(n - \tau\right)\overline{h}}{\mu \epsilon} \right\}.
\end{eqnarray} 

Minimizing $T(\tau)$ in $\tau$ is easy because
 $T(\tau)$ is a max of a linearly increasing term $\tau \times l(\tau)$ and a linearly decreasing term $\tau \times r(\tau)$ in $\tau$. Furthermore $n\times l(n) \geq 0 = n \times r(n) $. Consequently, if $ l(1) \geq r(1)$, then $\tau^* = 1$, otherwise
\begin{eqnarray}\label{eq:tau_optimal_independent}
 \tau^* = n\frac{\frac{2}{\mu \epsilon}\overline{h} -L_{\max}}{\frac{2}{\mu \epsilon}\overline{h} -L_{\max} +nL}.
\end{eqnarray} 
Since  $r(1)$ is proportional to the noise and $1/\epsilon$ and $l(1)$ is proportional to the smoothness constants the condition $l(1) \leq r(1)$ holds when there is comparatively a lot of noise or the precision is high. As we will see in Section~\ref{sec:zerogradnoise} this logic extends to the case where the noise is zero, where the optimal mini-batch size is $\tau^* =1.$

%
  
\textbf{$\tau$--nice sampling.}  Inserting the bound on $\cL$~\eqref{eq:cLns} and $\sigma$~\eqref{eq:sigmans} into~\eqref{eq:itercomplexlin} gives the  iteration complexity $k \geq \frac{2}{\mu} \max \{l(\tau), r(\tau) \}$, where
\begin{align}
l(\tau) &  = \frac{n (\tau -1)}{ \tau (n-1)}L  +\frac{n-\tau}{ \tau (n-1)}L_{\max},\label{eq:ltaunice} \\
r(\tau) & =   \frac{2(n-\tau)}{\epsilon \mu(n-1)}  \frac{\overline{h}}{\tau}, \label{eq:rtaunice}
\end{align} 
which holds for the stepsize 
\begin{equation} \label{eq:stepsize_tau}
\gamma =\frac{1}{2} \min \left\{\frac{1}{l(\tau)}, \frac{1}{r(\tau)} \right\} .
\end{equation} 
Again, this is an increasing function in $\tau.$
 
  We are now again able to  calculate the mini-batch size that optimizes the total complexity $T(\tau)$ given by 
$T(\tau) = \frac{2\tau}{\mu} \max \{l(\tau), r(\tau)\}.$ 
Once again 
 $T(\tau)$ is a max of a linearly increasing term $\tau \times l(\tau)$ and a linearly decreasing term $\tau \times r(\tau)$ in $\tau$. Furthermore $r(n) = 0 \leq l(n)$. Consequently, if $r(1)\leq   l(1)$ then $\tau^* = 1$, otherwise
 \begin{equation} \label{eq:tau_optimal_nice}
 \tau^* = n \frac{L  - L_{\max} +\frac{2}{\epsilon \mu} \cdot  \overline{h}}{nL  - L_{\max} +\frac{2}{\epsilon \mu} \cdot  \overline{h}}.
\end{equation}  

\subsection{Zero gradient noise}
\label{sec:zerogradnoise}

Consider the case where the gradient noise is zero ($\sigma =0$).
According to Theorem~\ref{theo:strcnvlin}, the resulting complexity of SGD with constant stepsize $\gamma = \frac{1}{2\cL}$ is given by the very simple expression
\begin{equation}\label{eq:la9sjjh8a3}
 k\geq   \frac{2\cL}{\mu },
\end{equation}
where we have dropped the logarithmic term $\log\left(\left.  \|x^0 - x^*\|^2 \right/  \epsilon \right)$.
In this setting, due to Corollary~\ref{cor:cLoverimpweak}, we know that $f$ satisfies the weak growth condition. Thus our results are directly comparable to those developed in~\cite{MaBB18} and in~\cite{vaswani2018fast}. 

In particular, Theorem~1 in~\cite{MaBB18} states that when running SGD with mini-batches based on sampling with replacement, the resulting iteration complexity is 
\begin{equation}\label{eq:mabb18}
 k \geq  \frac{L}{ \mu}\frac{\tau -1}{\tau}+\frac{L_{\max}}{ \mu}\frac{1}{\tau},
\end{equation}
again dropping the logarithmic term.
Now gaining insight into the complexity~\eqref{eq:la9sjjh8a3} is a matter of studying the expected smoothness parameter $\cL$ for different sampling strategies. 
%

%

\textbf{Independent sampling.}
Setting $\sigma =0$ (thus $\overline{h} =0$) and using 
uniform probabilities with $p_i = \frac{\tau}{n}$ in~\eqref{eq:itercomplexlinis}
gives 
\begin{eqnarray} \label{eq:totcomplexlinisuni2}
k &\geq  &   
\frac{2L}{\mu} +  \left(\frac{1}{\tau}-\frac{1}{n}\right)\frac{2L_{\max}}{\mu}.
\end{eqnarray}
 
\textbf{$\tau$ --nice sampling.}
If we use a uniform sampling and $\sigma =0$  then the resulting iteration complexity is given by 
\begin{eqnarray}\label{eq:itercomplexlinns2}
k &\geq &  
\frac{n (\tau -1)}{ \tau (n-1)}\frac{2L}{\mu}   +\frac{n-\tau}{ \tau (n-1)}\frac{2L_{\max}}{\mu}.
\end{eqnarray} 

Iteration complexities~\eqref{eq:mabb18}, \eqref{eq:totcomplexlinisuni2} and~\eqref{eq:itercomplexlinns2} tell essentially the same story. Namely, the complexity improves as $\tau$ increases to $n$, but this improvement is not enough when considering the total complexity (multiplying by $\tau$). Indeed, for total complexity, these results all say that $\tau =1$ is optimal.

\section{Importance Sampling}\label{sec:partialbiassamp}
In this section we propose importance sampling for single element sampling and independent sampling with $\mathbb{E}[|S|] = \tau$, respectively. Due to lack of space, the details of this section are in the appendix, Section \ref{sec:important}. Again we drop the log term in~(\ref{eq:itercomplexlin}) and adopt the notation in Remark~\ref{rem:LmaxLi}. 

\subsection{Single element sampling}
For single element sampling, plugging (\ref{eq:lmaxi1}) and (\ref{eq:sigma1}) into (\ref{eq:itercomplexlin}) gives the following iteration complexity 
$$
 \frac{2}{\epsilon \mu^2} \max\left\{ \frac{\epsilon \mu}{n}\max_{i\in [n]}\frac{L_i}{p_i} , \frac{2}{n^2}\sum_{i\in [n]}\frac{1}{p_i}\|h_i\|^2  \right\},
$$
where $0<p_i\leq 1$ and $\sum_{i\in [n]}p_i = 1$. In order to optimize this iteration complexity over $p_i$, we need to solve a $n$ dimensional linearly constrained nonsmooth convex minimization problem, which could be harder than the original problem (\ref{eq:prob}). So instead, we will focus on minimizing ${\cal L}_{\max}$ and $\sigma^2$ over $p_i$ seperately.
  We will then use these two resulting (sub)optimal probabilities to construct a sampling.

In particular, for single element sampling we can recover the {\em partially biased sampling} developed in~\cite{needell2014stochastic}. First, 
from~(\ref{eq:lmaxi1}) it is easy to see that the probabilities that minimize $\cL_{\max}$ are 
$
p_i^{\cal L} = \left. L_i \right/ \sum_{j\in [n]}L_j,
$
for all $i$.
 Using these suboptimal probabilities we can construct a partially biased sampling by letting $ \hat{p}_i \eqdef \frac{1}{2} p_i^{\cal L} +  \frac{1}{2n}.$
Plugging this sampling in~\eqref{eq:lmaxi1} gives 
${\cal L}_{\max} \leq 2\overline{L} \eqdef \frac{2}{n} \sum_{i \in [n]}L_i$, and from (\ref{eq:sigma1}), we have $\sigma^2 \leq \frac{2}{n}\sum_{i \in [n]} \|h_i\|^2 \eqdef 2 \overline{h}$. This sampling is the same as the partially biased sampling in ~\cite{needell2014stochastic}. From (\ref{eq:itercomplexlinis}) in Theorem \ref{theo:strcnvlin}, we get that the total complexity is now given by
\begin{equation}\label{eq:asdj88jaaa3}
 k \geq \max\left\{ \frac{4\overline{L}}{\alpha \mu}, \frac{8 \overline{h}}{\epsilon\mu^2} \right\}.
\end{equation}

For uniform sampling, ${\cal L}_{\max} = \max_{i\in [n]}L_i\geq \overline{L}$ and $\sigma^2 = \frac{1}{n}\sum_{i \in [n]} \|h_i\|^2$. Hence, compared to uniform sampling, the iteration complexity of partially biased sampling is at most two times larger, but could be $n/2$ smaller in the extreme case where 
$
L_{\max} = n \, \overline{L}.
$ 

\subsection{Minibatches}

Importance sampling for minibatches was first considered in \citep{csiba2016importance}; but not in the context of SGD. Here we propose the first importance sampling for minibatch SGD. In Section~\ref{sec:indepensamp} in the appendix we introduce the use of partially biased sampling together with independent sampling with $|S| = \tau$ and show that we can achieve a total complexity of (by Proposition \ref{pro:pialphapbs})
\begin{equation}\label{eq:asdj88jaaa23}
 k \geq \max\left\{ \left(1 - \frac{2}{\tau}\right)\frac{2\overline{L}}{\alpha \mu}, \left(\frac{2}{\tau} - \frac{1}{n}\right)\frac{8 \overline{h}}{\epsilon\mu^2} \right\},
\end{equation}
which not only eliminates the dependence on $L_{\max}$, but also improves as the mini-batch size $\tau$ increases. 

\section{Experiments }
\label{exper}
In this section, we empirically validate our theoretical results. We perform three experiments  in each of which we  highlight a different aspect of our contributions. 

In the first two experiments we focus on ridge regression and regularized logistic regression problems (problems with strongly convex objective $f$ and components $f_i$) and we evaluate the performance of SGD on both synthetic and real data. 
In the second experiment (Section~\ref{sec:exp2}) we compare the convergence of SGD for several choices of the distribution $\cD$ (different sampling strategies) as described in Section~\ref{sec:samplingD}. In the last experiment (Section~\ref{sec:exp3}) we focus on the problem of principal component analysis (PCA) which by construction can be seen as a problem with a strongly convex objective $f$ but with non-convex functions $f_i$ \cite{allen2016improved, garber2015fast,  shalev2016sdca}.

In all experiments, to evaluate SGD we use the relative error measure  $\frac{\|x^k - x^* \|^2}{\|x^0 - x^* \|^2}$. For all implementations, the starting point $x^0$ is sampled from the standard Gaussian. We run each method until $\|x^k-x^* \|^2\leq 10^{-3}$ or until a pre-specified maximum number of epochs is achieved. For the horizontal axis we always use the number of epochs. 

For more experiments we refer the interested reader to Section~\ref{addExp} of the Appendix. 

\textbf{Regularized Regression Problems:} 
In the case of the \emph{ridge regression} problem we solve:
$$\min_x  f(x) =\frac{1}{2 n} \sum_{i=1}^{n} ( \bA[i, :] x - y_i)^2 + \frac{\lambda}{2} \| x \|^2,
$$
while for the \emph{$L2$-regularized logistic regression} problem we solve:
$$
\min_x  f(x) = \frac{1}{2 n} \sum_{i=1}^{n} \log \left(1+\exp(-y_i \bA[i, :] x) \right) + \frac{\lambda}{2} \| x \|^2 .
$$

In both problems $\mathbf{A} \in \mathbb{R}^{n \times d}, y \in \mathbb{R}^n$ are the given data and $\lambda > 0$ is the regularization parameter. 
We generated synthetic data in both problems by sampling the rows of matrix $\mathbf{A}$ ($\bA[i, :]$)  from the standard Gaussian distribution $\mathcal{N}(0, 1)$. Furthermore 
for ridge regression we sampled the entries of $y$ from the standard Gaussian distribution while in the case of logistic regression $y \in \{-1, 1\}^n$ where $\mathbb{P}(y_i=1) = \mathbb{P}(y_i=-1) = \frac{1}{2}$.
For our experiments on real data we choose several LIBSVM \cite{chang2011libsvm} datasets.

\begin{figure}[t]
\centering
\begin{subfigure}{.24\textwidth}
  \centering
  \includegraphics[width=1\linewidth]{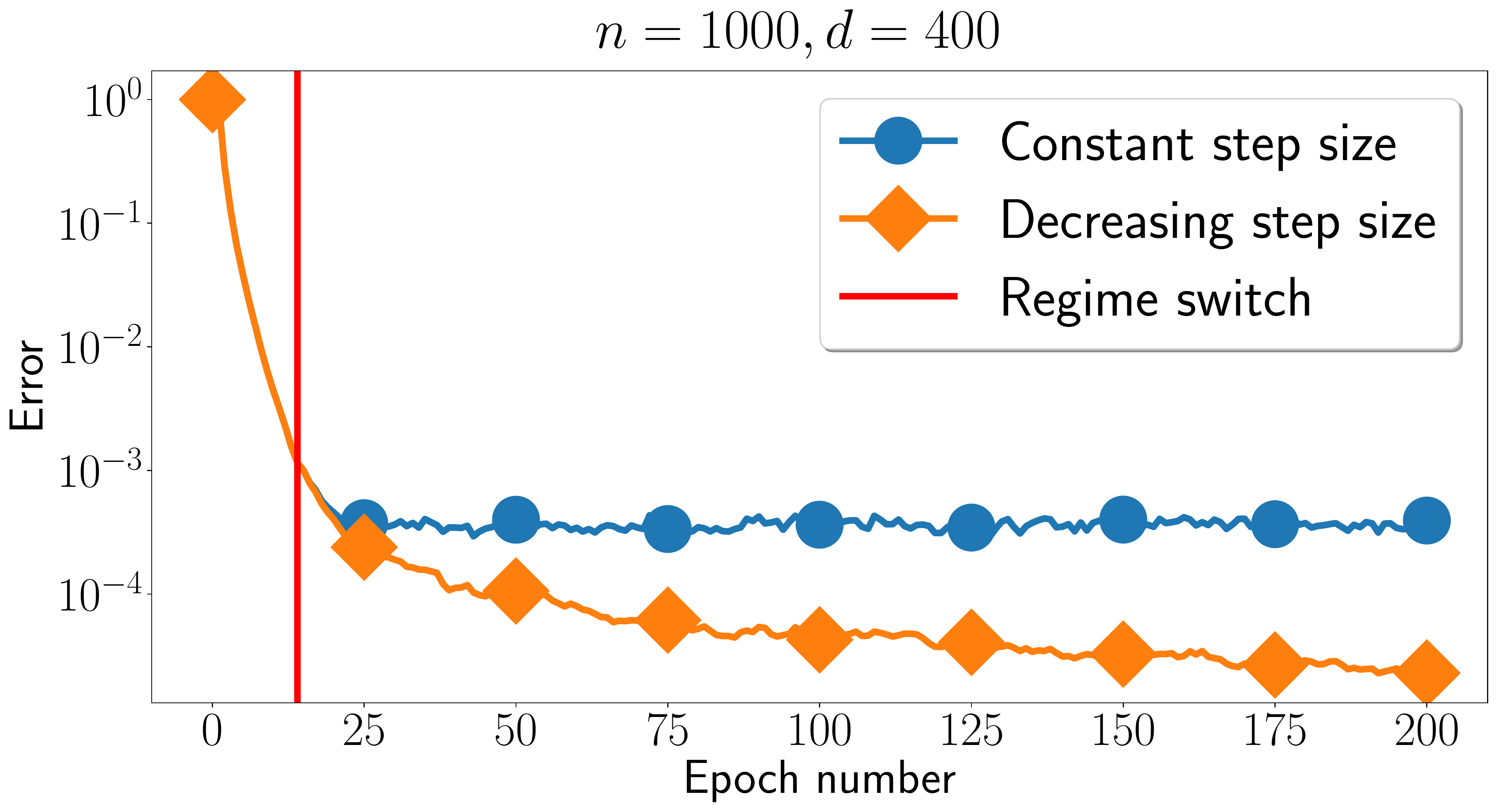}
\end{subfigure}%
\begin{subfigure}{.24\textwidth}
  \centering
  \includegraphics[width=1\linewidth]{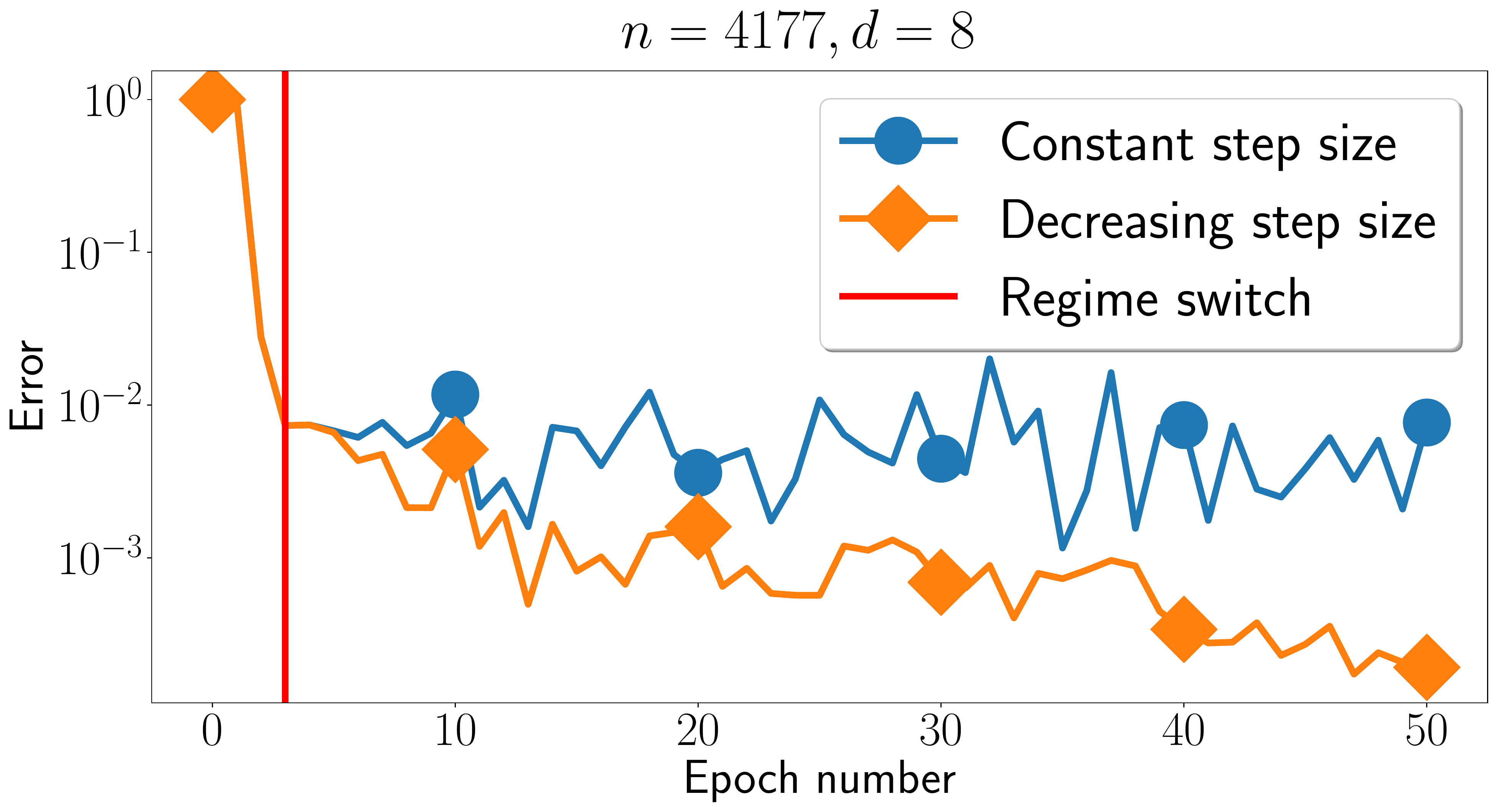}
\end{subfigure}\\
\begin{subfigure}{.24\textwidth}
  \centering
  \includegraphics[width=1\linewidth]{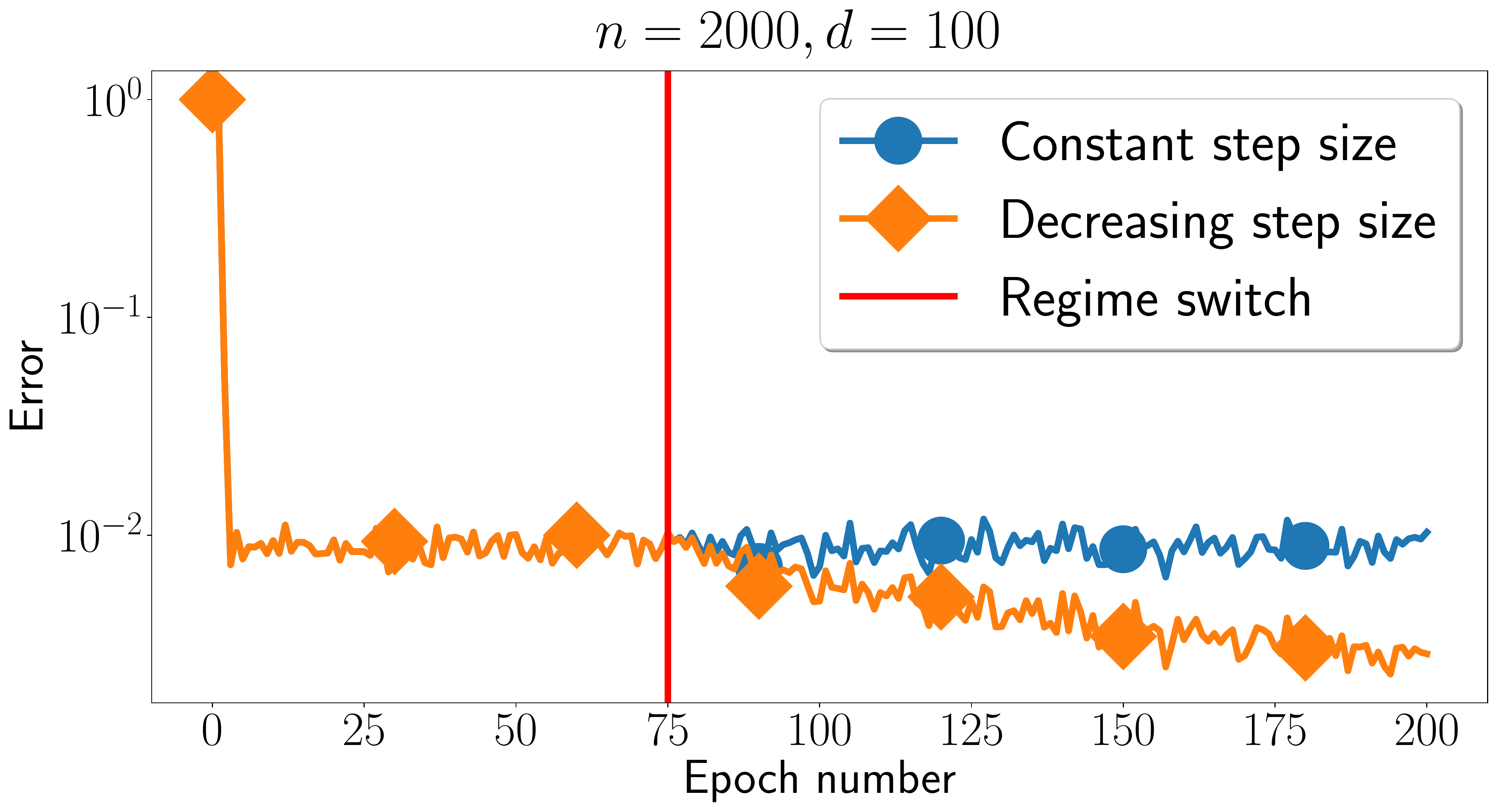}
\end{subfigure}%
\begin{subfigure}{.24\textwidth}
  \centering
  \includegraphics[width=1\linewidth]{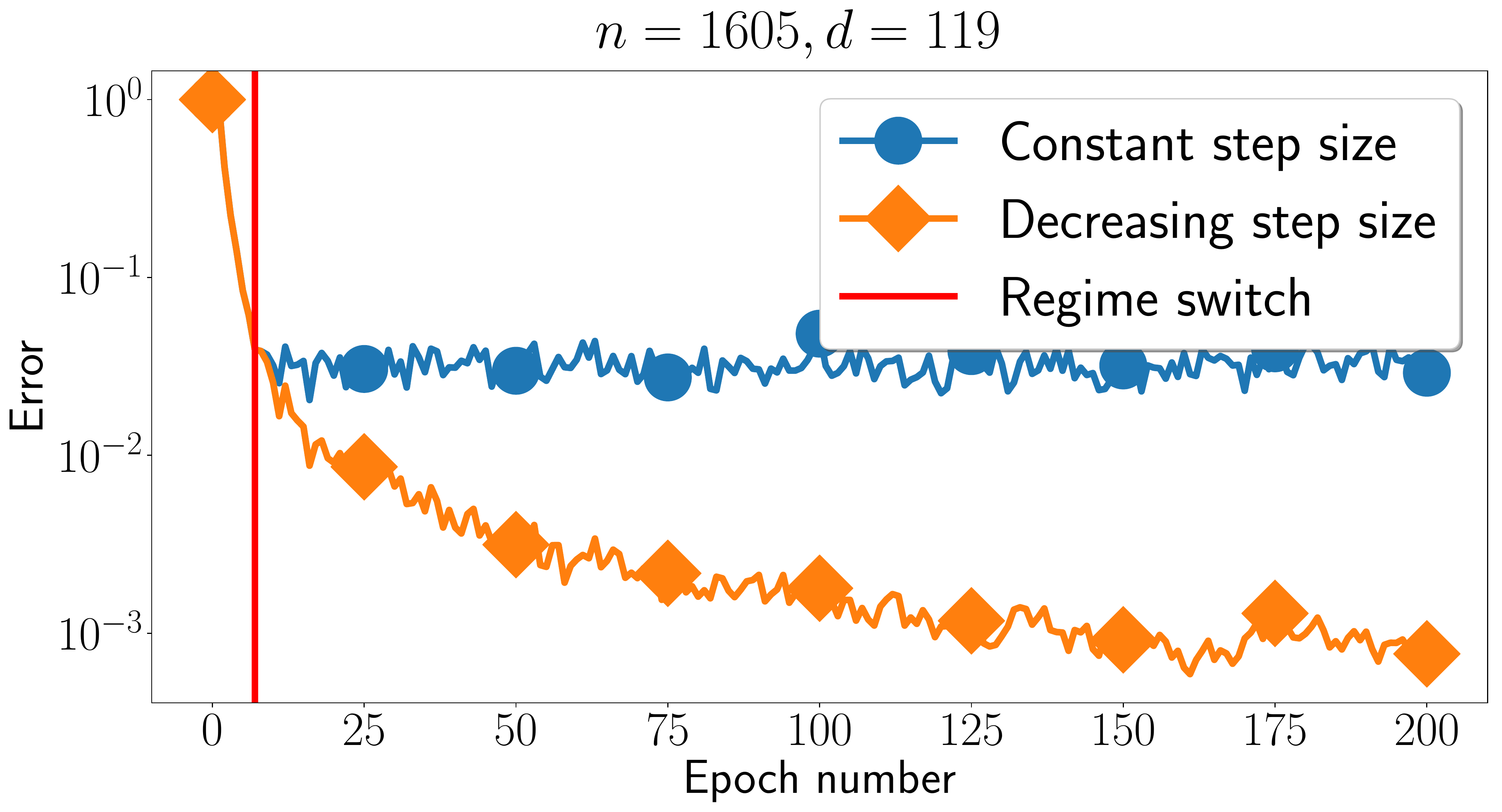}
\end{subfigure}
\caption{\footnotesize Comparison between constant and decreasing step size regimes of SGD. Ridge regression problem (first row): on left - synthetic data, on right - real dataset: abalone from LIBSVM.  Logistic regression problem(second row): on left - synthetic data, on right - real data-set: a1a from LIBSVM. In all experiments $\lambda = 1/n$.
}
\label{constantVsdecreasing}
\end{figure}

\subsection{Constant vs decreasing step size} \label{sec:exp1}

We now  compare the performance of SGD in the constant and decreasing stepsize regimes considered in  Theorems \ref{theo:strcnvlin} (see \eqref{eq:stepbndmax}) and  \ref{theo:decreasingstep} (see \eqref{eq:gammakdef}), respectively. Here we use
a uniform single element sampling.
 As expected from theory, we see in  Figure~\ref{constantVsdecreasing} that the decreasing stepsize regime is vastly superior at reaching a higher precision than the constant step-size variant. In our plots, the vertical red line denotes the value of $4\lceil\mathcal{\left.\cL\right/\mu}\rceil$ predicted from Theorem~\ref{theo:decreasingstep} and highlights the point where SGD needs to change its update rule from constant to decreasing step-size.

\begin{figure}[t] 
\centering
\begin{subfigure}{.45\textwidth}
  \centering
  \includegraphics[width=1\linewidth]{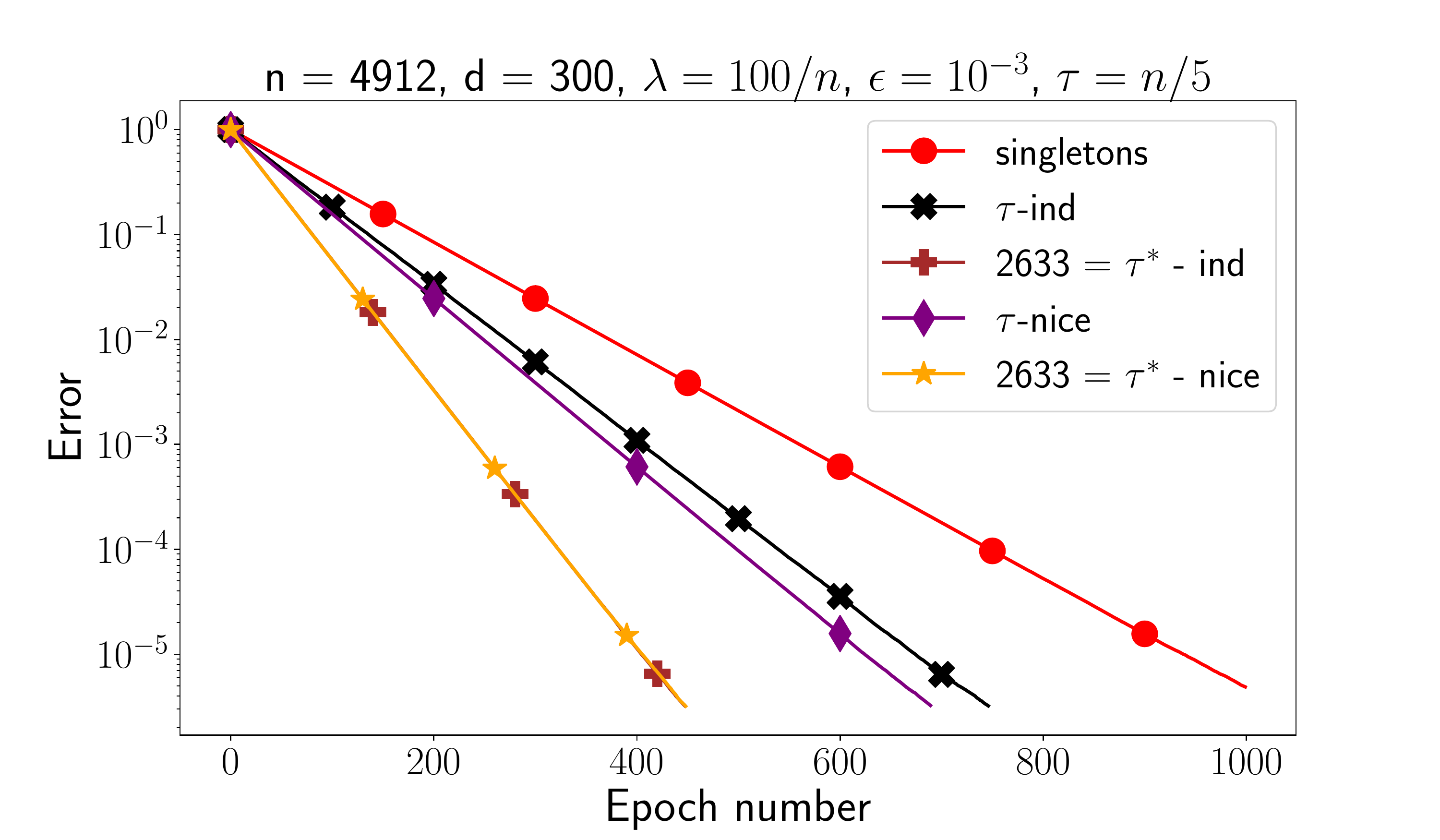}
\end{subfigure}  \\
\begin{subfigure}{.45\textwidth}
  \centering
  \includegraphics[width=1\linewidth]{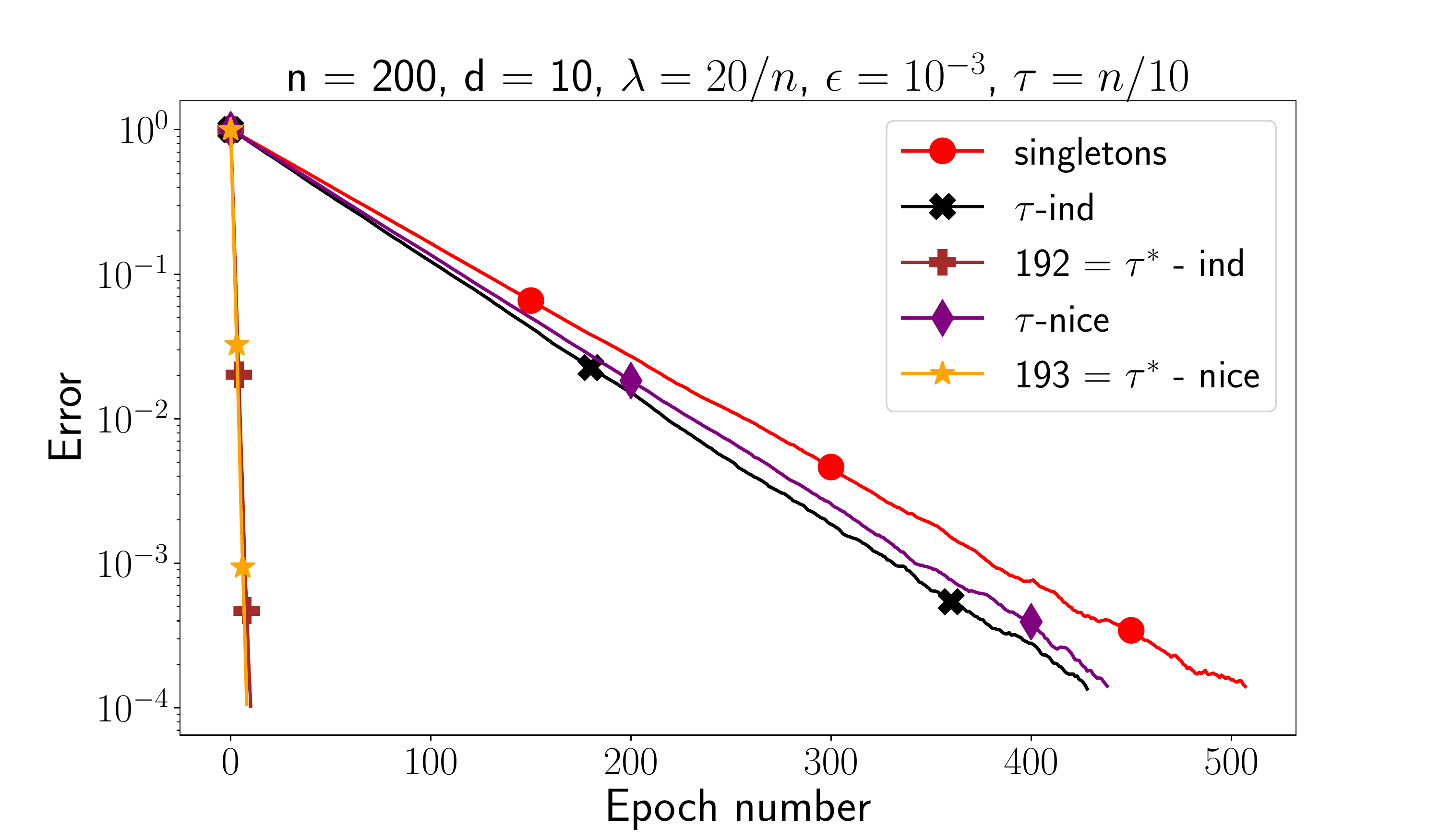}
\end{subfigure}
\caption{\footnotesize Performance of SGD with several minibatch strategies for logistic regression. Above: the w3a data-set from LIBSVM. Below: standard Gaussian data. }
\label{FigSamplings}
\end{figure}

\subsection{Minibatches} \label{sec:exp2}

In Figures~\ref{FigSamplings} and \ref{FigSamplingsXXX} we compare the single element sampling (uniform and importance), $\tau$ independent sampling (uniform, uniform with optimal batch size and importance) and $\tau$ nice sampling (with some $\tau$ and with optimal $\tau^*$).
The probabilities of importance samplings in the single element sampling and $\tau$ independent sampling are calculated by formulas (\ref{eq:optimal_prob_sing}) and (\ref{eq: optimal_prob_independ}) in the Appendix. 
Formulas for optimal minibatch size $\tau^*$ in independent sampling and $\tau$-nice samplings are given in (\ref{eq:tau_optimal_independent}) and (\ref{eq:tau_optimal_nice}), respectively. Observe that minibatching with optimal $\tau^*$ gives the best convergence. In addition, note that for constant step size, the importance sampling variants depend on the accuracy $\epsilon$. From Figure~\ref{FigSamplings} we can see that before the error reaches the required accuracy, the importance sampling variants are comparable or better than their coresponding uniform sampling variants. 

\subsection{Sum-of-non-convex functions}\label{sec:exp3}

In Figure~\ref{FigPCA}, our goal is to illustrate that Theorem~\ref{theo:strcnvlin} holds even if the functions $f_i$ are non convex. This experiment is based on the experimental setup given in~\cite{allen2016improved}.
We first generate random vectors $a_1, \dots, a_n, b \in \mathbb{R}^d$ from $\mathcal{U}(0, 10)$  and set
$ \mathbf{A} \eqdef \frac{1}{n} \sum_{i=1}^n a_i a_i^\top $. Then we consider the problem:
\[
\min_x  f(x) =\frac{1}{2n} \sum_{i=1}^n x^\top (a_i a_i^\top + D_i) x  + b^\top x,
\]
where $D_i,$ $i \in [n]$ are diagonal matrices satisfying $D\eqdef D_1 + \cdots + D_n = 0$. 
In particular, to guarantee that $D=0$, we randomly select half of the matrices and assign their $j$-th diagonal value
$(D_i)_{jj}$ equal to $11$; for the other half we assign $(D_i)_{jj}$ to be $-11$. We repeat that for all diagonal values.
Note that under this construction, each $f_i$ is a non-convex function.
Once again,  in the first plot we observe that while both are equally fast in the beginning, the decreasing stepsize variant is better at reaching higher accuracy than the fixed stepsize variant. In the second plot we see, as expected, that all four minibatch versions of SGD outperform single element SGD. However, while  the $\tau$-nice and $\tau$-independent samplings with $\tau=n/5$ lead to a slight improvement only, the  theoretically optimal choice $\tau=\tau^*$ leads to a vast improvement. 

\begin{figure}[th]
\centering
\begin{subfigure}{.45\textwidth}
  \centering
  \includegraphics[width=1\linewidth]{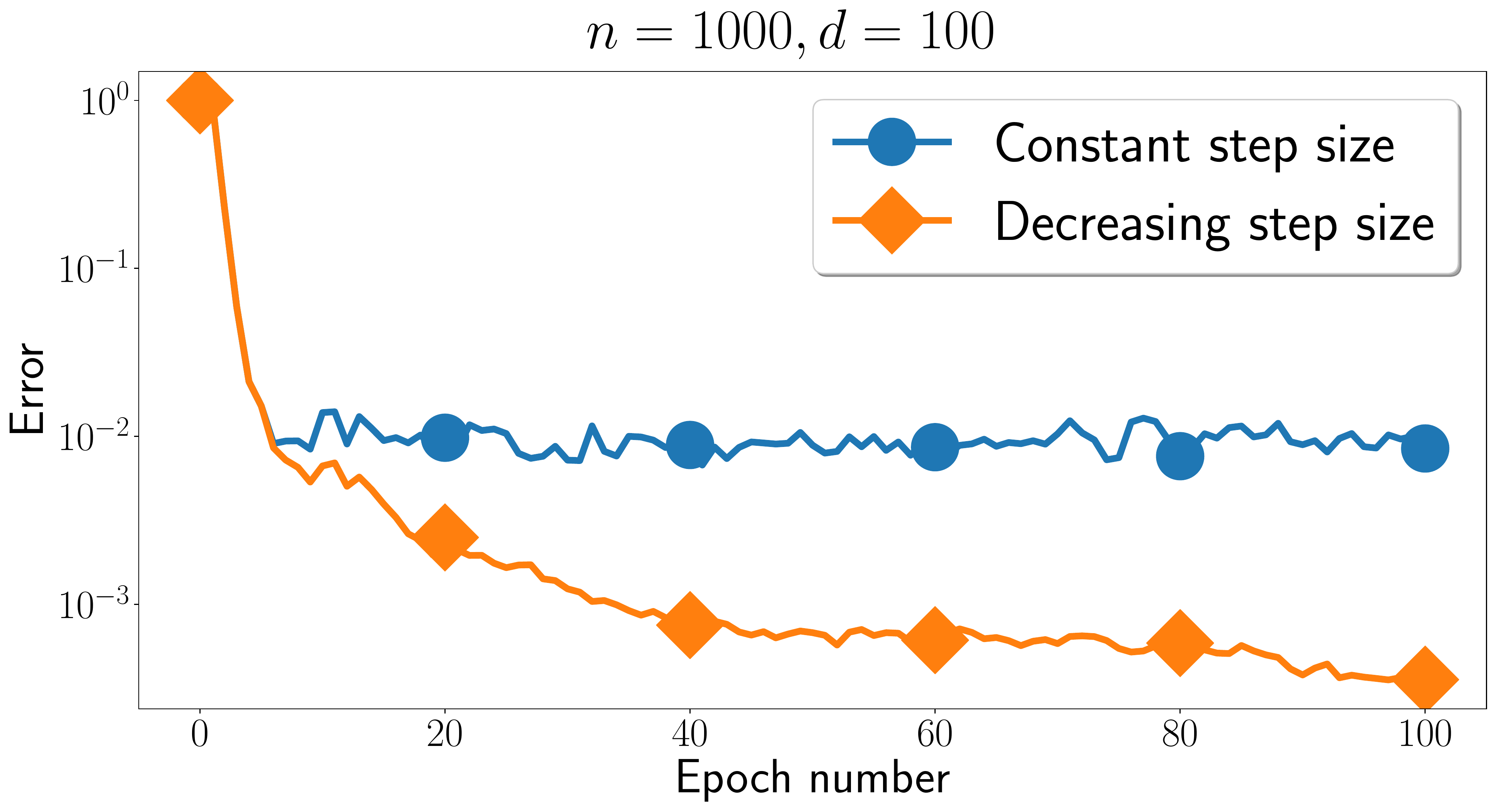}
\end{subfigure}\\
\begin{subfigure}{.45\textwidth}
  \centering
  \includegraphics[width=1\linewidth]{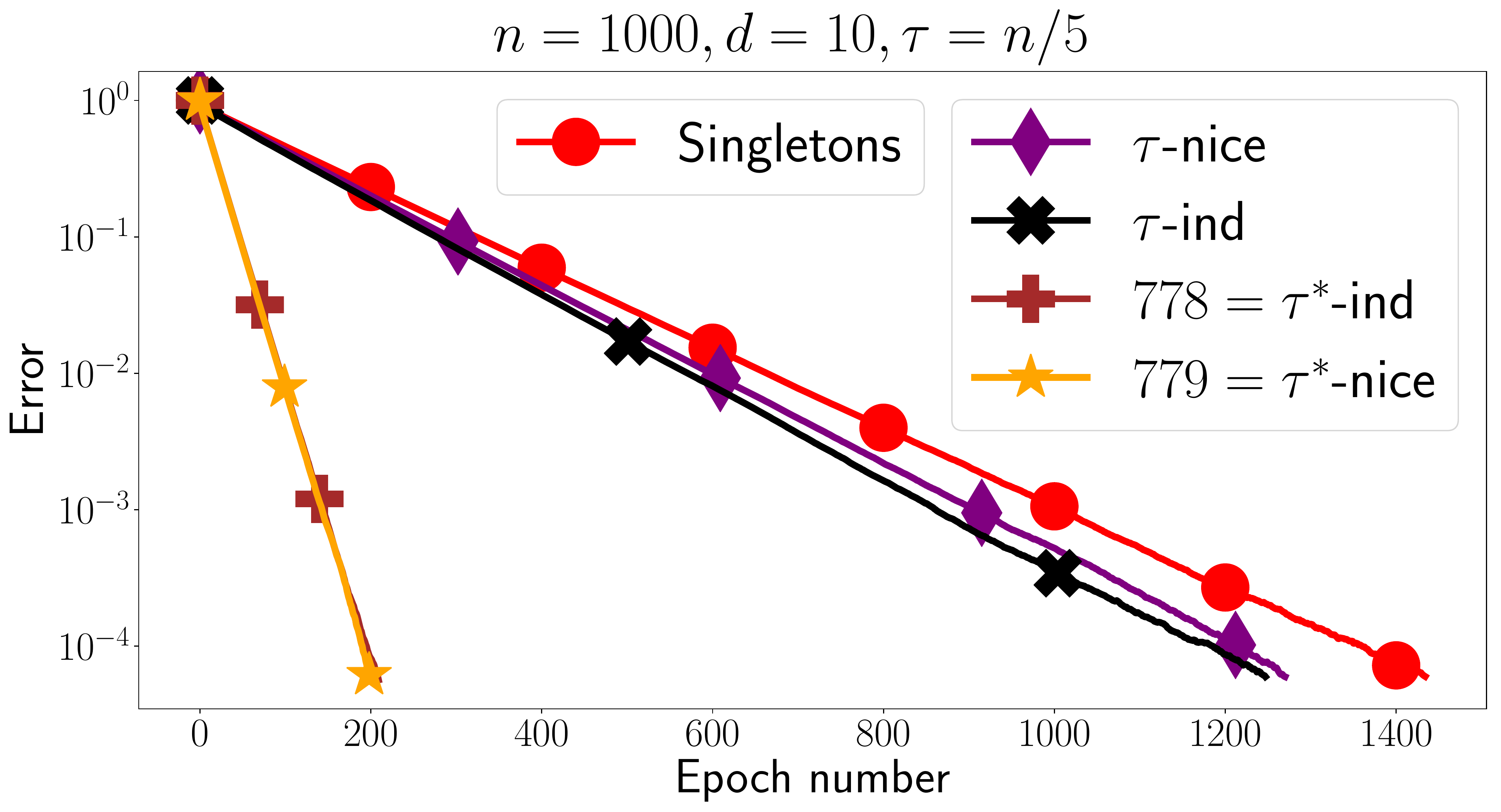}
\end{subfigure}\\
\caption{Above: Comparison between constant and decreasing step size regimes of SGD for PCA. Below:  comparison of different sampling strategies of SGD for PCA.}
\label{FigPCA}
\end{figure}


\section*{Acknowledgements}
RMG acknowledges the support by a public grant as part of the Investissement d'avenir project, reference ANR-11-LABX-0056-LMH, LabEx LMH, in a joint call with Gaspard Monge Program for optimization, operations research and their interactions with data sciences.

{
\footnotesize
\bibliographystyle{icml2019blank}
\bibliography{biblio1}
}

\appendix 

\onecolumn
\icmltitle{APPENDIX \\ SGD: General Analysis and Improved Rates}

\section{Elementary Results}

In this section we collect some elementary results; some of them we use repeatedly. 

\begin{proposition}\label{prop:ihs9hd} Let $\phi:\R^d \to \R$ be $L_\phi$--smooth, and assume it has a minimizer $x^*$ on $\R^d$. Then \[\|\nabla \phi(x) - \nabla \phi(x^*)\|^2 \leq 2 L_\phi (\phi(x)-\phi(x^*)).\] 
\end{proposition}
\begin{proof} Lipschitz continuity of the gradient implies that
\[\phi(x+h) \leq \phi(x) + \dotprod{\nabla \phi(x), h} + \frac{L_\phi}{2}\|h\|^2.\]
Now plugging $h=-\frac{1}{L_\phi} \nabla \phi(x)$ into the above inequality, we get $\frac{1}{2 L_\phi} \|\nabla \phi(x)\|^2 \leq \phi(x)-\phi(x+h) \leq  \phi(x)-\phi(x^*)$. It remains to note that $\nabla \phi(x^*)=0$.
\end{proof}

In this section we summarize some elementary results which we  use often in our proofs. We do not claim novelty; we but we include them for completeness and clarity.

\begin{lemma}[Double counting]\label{lem:2count}
Let $a_{i,C} \in \R$ for $i=1,\ldots, n$ and $C \in \cC$, where $\cC$ is some collection of subsets of $[n]$. Then
\begin{equation}
\sum_{C \in \cC} \sum_{i\in C} a_{i,C} \quad 
=\quad   \sum_{i=1}^n \sum_{C \in \cC \; : \; i \in C} a_{i,C}.
\end{equation}
\end{lemma}

\begin{lemma}[Complexity bounds] \label{lem:complexbndlog}
Let $E>0$, $\;0< \rho \leq 1$ and $ 0 \leq c<1$.
If $k \in \mathbb{N}$ satisfies
\begin{equation}\label{eq:oak938jfse}
 k \geq \frac{1}{1-\rho} \log\left(\frac{ E}{ (1-c)  }\right)  , 
\end{equation}
then
\begin{equation} \label{eq:acb378babuhi3A3}
\rho^k \leq  (1-c) E.
\end{equation}
\end{lemma}

\begin{proof}
Taking logarithms and rearranging~\eqref{eq:acb378babuhi3A3}  gives
\begin{equation}
  \log\left(\frac{ E}{ 1-c  }\right) \leq k\log\left(\frac{1}{\rho}\right).\label{eq:acb378babuhi3A32}
\end{equation}
Now using that $\log\left(\frac{1}{\rho}\right) \geq 1-\rho,$ for $0<\rho \leq 1$ gives~\eqref{eq:oak938jfse}. 
\end{proof}

\subsection{The iteration complexity~\eqref{eq:itercomplexlin} of Theorem~\ref{theo:strcnvlin} } \label{sec:itercomplextheo}

To analyse the iteration complexity, let $\epsilon >0$ 
and choosing the stepsize so that
$
\frac{2\gamma \sigma^2}{\mu} \leq \frac{1}{2}\epsilon,
$ 
gives~\eqref{eq:stepbndmax}.
Next we choose $k$ so that 
\[\left( 1 - \gamma \mu \right)^k \|r^0\|^2 \leq \frac{1}{2}\epsilon. \]
Taking logarithms and re-arranging the above gives
\begin{equation}
  \log\left(\frac{2 \|r^0\|^2}{  \epsilon }\right) \leq k\log\left(\frac{1}{1 - \gamma \mu }\right).\label{eq:acb378babuhi3}
\end{equation}
Now using that $\log\left(\frac{1}{\rho}\right) \geq 1-\rho,$ for $0<\rho \leq 1$ gives
\begin{eqnarray}
k &\geq & \frac{1}{\gamma\mu }  \log\left(\frac{ 2 \|r^0\|^2 }{  \epsilon }\right) \nonumber \\
& \overset{\eqref{eq:stepbndmax}}{=} & \frac{1}{\mu } \max \left\{ 2\cL,\; \frac{4 \sigma^2}{\epsilon\mu}\right\} \log\left(\frac{ 2 \|r^0\|^2 }{  \epsilon }\right).
\end{eqnarray}
Which concludes the proof.

\section{Proof of Lemma~\ref{lem:weakgrowth}}

For brevity, let us write
 $\Exp[\cdot]$ instead of $\Exp_{\cD}[\cdot]$. Then
 \begin{align*}
\Exp \|\nabla f_{v} (x)\|^2 &= \Exp \| \nabla f_v(x) - \nabla f_{v}(x^*) + \nabla f_{v}(x^*) \|^2 \\
&\vspace{-1cm}\leq  2 \Exp \| \nabla f_{v}(x) - \nabla f_{v}(x^*)\|^2 
 + 2  \Exp \|\nabla f_{v}(x^*)\|^2 \notag\\ 
&\leq 4 \cL [f(x)-f(x^*)] + 2 \Exp \|\nabla f_{v}(x^*)\|^2.
\end{align*}
The first inequality follows from the estimate $\|a+b\|^2 \leq 2\|a\|^2 + 2\|b\|^2$, and the second inequality follows from \eqref{eq:expsmooth}.

\section{Proof of Theorem~\ref{theo:decreasingstep} }

\begin{proof}
Let  $\gamma_k \eqdef \frac{2k+1}{(k+1)^2 \mu}$ and let $k^*$ be an integer that satisfies
$
\gamma_{k^*} \leq \frac{1}{2\cL}.
$ 
In particular this holds for
 \[k^* \geq  \lceil 4\mathcal{K} -1\rceil.\] 
Note that $\gamma_k$ is decreasing in $k$ and  consequently $\gamma_k \leq \frac{1}{2\cL}$ for all $k \geq k^*.$ This in turn guarantees that~\eqref{eq:la9ja38jf}  holds for all $k\geq k^*$ with $\gamma_k$ in place of $\gamma$, that is
\begin{equation}
\mathbb{E}\| r^{k+1}\|^2 \leq \frac{k^2}{(k+1)^2}\mathbb{E} \|r^k\|^2 + \frac{2\sigma^2}{\mu^2}\frac{(2k+1)^2}{(k+1)^4 }.
\end{equation}
Multiplying both sides by $(k+1)^2$ we obtain
\begin{eqnarray*}
(k+1)^2 \mathbb{E}\| r^{k+1}\|^2 &\leq & 
k^2 \mathbb{E} \|r^k\|^2 + \frac{2\sigma^2}{\mu^2} \left(\frac{2k+1}{k+1}\right)^2 \\
 &\leq & k^2 \mathbb{E} \|r^k\|^2 + \frac{8 \sigma^2}{\mu^2},
\end{eqnarray*}
where the second inequality holds because  $\frac{2k+1}{k+1} <2$. Rearranging and summing from $t= k^* \ldots k$ we obtain:
\begin{equation}
\sum_{t=k^*}^{k} \left[ (t+1)^2 \mathbb{E}\| r^{t+1}\|^2 - t^2 \mathbb{E} \|r^t\|^2 \right] \leq  \sum_{t=k^*}^{k} \frac{8 \sigma^2}{\mu^2}. 
\end{equation}
Using telescopic cancellation gives
\[
(k+1)^2 \mathbb{E}\| r^{k+1}\|^2 \leq  (k^*)^2 \mathbb{E} \|r^{k^*}\|^2 +\frac{8 \sigma^2 (k-k^*)}{\mu^2}.
\]
Dividing the above by $(k+1)^2$ gives
\begin{equation}
 \mathbb{E}\| r^{k+1}\|^2 \leq  \frac{(k^*)^2}{(k+1)^2 } \mathbb{E} \|r^{k^*}\|^2 +\frac{8 \sigma^2 (k-k^*)}{\mu^2(k+1)^2 }. \label{eq:cndsiu48js}
\end{equation}
For $k \leq k^*$ we have that~\eqref{eq:la9ja38jf} holds, which combined with~\eqref{eq:cndsiu48js}, gives 
\begin{eqnarray}
 \mathbb{E}\| r^{k+1}\|^2 &\leq &
  \frac{(k^*)^2}{(k+1)^2 } \left( 1 -  \frac{\mu}{2\cL} \right)^{k^*} \|r^{0}\|^2 \nonumber \\ &   
  +&\frac{\sigma^2 }{\mu^2 (k+1)^2}\left(8 (k-k^*) +   \frac{(k^*)^2}{\mathcal{K} } \right).  \label{eq:sdaiuna3}
\end{eqnarray}
Choosing $k^*$ that minimizes the second line of the above gives $k^* = 4\lceil\mathcal{K} \rceil$, which when inserted into~\eqref{eq:sdaiuna3} becomes
\begin{eqnarray}
 \mathbb{E}\| r^{k+1}\|^2 &\leq &
  \frac{16 \lceil\mathcal{K} \rceil^2}{(k+1)^2 } \left( 1 -  \frac{1}{2\mathcal{K}} \right)^{ 4\lceil\mathcal{K} \rceil} \|r^0\|^2  \nonumber \\
 & & +\frac{\sigma^2 }{\mu^2 }\frac{8 (k-2\lceil\mathcal{K} \rceil)}{(k+1)^2} \nonumber \\
  & \leq &  \frac{16 \lceil\mathcal{K} \rceil^2}{e^2(k+1)^2 }  \|r^{0}\|^2  +  \frac{\sigma^2 }{\mu^2 }\frac{8 }{k+1}, \label{eq:sdaiuna32}
\end{eqnarray}
where we have used that $\left( 1 -  \frac{1}{2x} \right)^{ 4x} \leq e^{-2}$  for all $x \geq 1.$

\end{proof}

\section{Proof of Theorem~\ref{theo:expsmoothas}}

\begin{proof}
Since $v_i=v_i(S) = \mathbf{1}_{(i\in S)}\frac{1}{p_i}$.
and since  $f_i$ is $\mM_i$-smooth, the function
	\begin{equation}\label{eq:fvas1}
	f_v(x) = \frac{1}{n} \sum_{i=1}^n f_i(x)v_i = \frac{1}{n} \sum_{i \in S}\frac{f_i(x)}{p_i},
	\end{equation}
 is $L_S$--smooth where
 \[L_S\eqdef  \frac{1}{n}\lambda_{\max}\left(\sum_{i \in S} \frac{\mM_i}{p_i} \right).\] 

We also define the following smoothness related quantities
\begin{equation}\label{eq:smoothnessquants}
 \cL_i \eqdef \sum_{C\;:\; i\in C} \frac{p_C}{p_i} L_C,
  \;\quad \cL_{\max} \eqdef \max_i \cL_i, \;\mbox{and}
 ;\quad L_{\max} = \max_{i \in [n]} \lambda_{\max}(\mM_i).
\end{equation}
	Since the $f_i$'s are convex and the sampling vector $v \in \R^d_+$ has positive elements, each realization of $f_v$ is convex and smooth, thus it follows from equation (2.1.7) in Theorem 2.1.5 in~\cite{nesterov2013introductory}  that 
	\begin{equation}\label{eq:fvas3}
	\|\nabla f_v(x) - \nabla f_v(y)\|^2 \leq 2L_S\left( f_v(x) - f_v(y) - \langle \nabla f_v(y), x-y \rangle \right).
	\end{equation}
	Taking expectation in (\ref{eq:fvas3}) gives 
	\begin{eqnarray*}
		\mathbb{E} [\|\nabla f_v(x) - \nabla f_v(y)\|^2 ] &\leq& 2\sum_{C} p_CL_C \left( f_{v(C)}(x) - f_{v(C)}(y) - \langle \nabla f_{v(C)}(y), x-y \rangle \right) \\
		&\overset{\eqref{eq:fvas1}}{=}& 2\sum_C p_C L_C \sum_{i \in C} \frac{1}{np_i} \left( f_i(x) - f_i(y) - \langle \nabla f_i(y), x-y \rangle \right) \\
		&\overset{\scriptsize \mbox{Lemma}~\ref{lem:2count} }{=}& \frac{2}{n}\sum_{i=1}^n \sum_{C: i\in C} p_C\frac{1}{p_i} L_C \left( f_i(x) - f_i(y) - \langle \nabla f_i(y), x-y \rangle \right) \\
		&\overset{(\ref{eq:expsmoothas})}{\leq}& \frac{2}{n} \sum_{i=1}^n \cL_{\max} \left( f_i(x) - f_i(y) - \langle \nabla f_i(y), x-y \rangle \right) \\
		&=& 2{\cal L}_{\max} \left( f(x) - f(y) - \langle \nabla f(y), x-y \rangle  \right).
	\end{eqnarray*}
	
Furthermore, for each $i$, 
\begin{eqnarray}
\cL_i =	\sum_{C: i\in C}\frac{p_C}{p_i}L_C  &=& \frac{1}{n}\sum_{C: i\in C}\frac{p_C}{p_i}\lambda_{\max}\left(\sum_{j\in C}\frac{\mM_j}{p_j} \right) \label{eq:tempcalcneed}\\
	&\leq&  \frac{1}{n} \sum_{C:i\in C}\frac{p_C}{p_i} \sum_{j\in C} \frac{\lambda_{\max}(\mM_j)}{p_j} \nonumber \\
	&\overset{\text{Lemma}~\ref{lem:2count} }{=}& \frac{1}{n} \sum_{j=1}^n \sum_{C: i\in C \ \& \ j\in C} \frac{p_C}{p_ip_j} \lambda_{\max}(\mM_j)  \nonumber\\
	&=& \frac{1}{n} \sum_{j=1}^n\frac{\mP_{ij}}{p_ip_j} \lambda_{\max}(\mM_j).  \nonumber
\end{eqnarray}
Hence, 
\begin{equation}\label{eq:lmaxiupbound}
\cL_{\max} \leq \frac{1}{n} \max_{i\in [n]} \left\{ \sum_{j\in [n]}\mP_{ij}\frac{\lambda_{\max}(\mM_j)}{p_ip_j} \right\}. 
\end{equation}

Let $y=x^*$ and notice that $\nabla f(x^*) = 0$, which gives~\eqref{eq:expsmoothas}.
We prove~\eqref{eq:cLmaxbound} in the following slightly more comprehensive Lemma~\ref{lem:98s9g8s}.
\end{proof}

\section{Bounds on the Expected Smoothness Constant $\cL$}

Below we establish some lower and upper bounds on the expected smoothness constant $\cL = \cL_{\max}$. These bounds were referred to in the main paper in Section~\ref{sec:bifg79f9d}. We also make use of notation introduced in Section~\ref{sec:cLandsigma}.

\begin{lemma}\label{lem:98s9g8s} Assume that there exists $\tau\in [n]$ such that $|S|=\tau$ with probability 1. Let 
\[\cL_i \eqdef \E{L_S \;|\; i\in S} = \sum_{C\;:\; i\in C} \frac{p_C}{p_i} L_C,\]
 and
\[\bar{\cL}_S \eqdef \frac{1}{|S|}\sum_{i\in S} \cL_i.\]
 Then $\E{\bar{\cL}_S} = \E{L_S}$. Moreover,
\begin{equation}\label{eq:cLmaxbigbound}
 L \leq \E{\bar{\cL}_S} \leq \cL_{\max} \leq L_{\max}.\end{equation}
\end{lemma}
\begin{proof}


Define $\mM_S \eqdef \frac{1}{n}\sum_{i\in S} \frac{\mM_i}{p_i}$ and note that $f$ is $\frac{1}{n}\sum_{i \in [n]} \mM_i$--smooth.
 Furthermore
\[\E{\mM_S} = \frac{1}{n}\E{ \sum_{i=1}^n \frac{\mM_i}{p_i} \mathbf{1}_{(i\in S)}} =
\frac{1}{n} \sum_{i=1}^n \frac{\mM_i}{p_i} \E{\mathbf{1}_{(i\in S)}} = \frac{1}{n}\sum_{i \in [n]} \mM_i.\]
We will now establish the inequalities in~\eqref{eq:cLmaxbigbound} starting from left to the right.

{\bf (Part I $L \leq \E{L_S}$).} Recalling that $L_S = \lambda_{\max}(\mM_S)$ and by Jensen's inequality,
\[L=\lambda_{\max}\left(\E{\mM_S}\right) \leq \E{\lambda_{\max}(\mM_S)} = \E{L_S}.\]
Furthermore
\begin{align*}
\E{\bar{\cL}_S} &= \E{\frac{1}{\tau}\sum_{i\in S} \cL_i} = \frac{1}{\tau}\sum_i p_i \cL_i \\
& \overset{\eqref{eq:smoothnessquants}}{=} \frac{1}{\tau}\sum_i \sum_{C\;:\; i\in C} p_C L_i  \overset{\text{Lemma}~\ref{lem:2count} }{=} \frac{1}{\tau}\sum_{C} \sum_{i\in C} p_C L_C\\
& =\frac{1}{\tau}\sum_{C} |C| p_C L_C=\sum_{C}  p_C L_C = \E{L_S}
\end{align*}

{\bf (Part II $\E{\bar{L}_S} \leq \cL_{\max}$).} We have that
\[\bar{L}_S =  \frac{1}{|S|}\sum_{i\in S} \cL_i \leq \frac{1}{|S|} \sum_{i\in S}\max_{i\in [n]}\cL_i = \cL_{\max}. \]

{\bf (Part III $\cL_{\max} \leq L_{\max}$).} Finally, since 
\begin{equation}\label{eq:LCbound}
 L_C \leq \frac{1}{\tau}\sum_{j \in C} L_j \leq L_{\max},\end{equation}
we have that
\[ \cL_i \overset{\eqref{eq:smoothnessquants}+\eqref{eq:LCbound} }{\leq}  \sum_{C\;:\; i\in C} \frac{p_C}{p_i} \frac{1}{\tau}\sum_{j \in C} L_j \overset{\eqref{eq:LCbound} }{\leq} 
 \sum_{C\;:\; i\in C} \frac{p_C}{p_i} L_{\max} = L_{\max}. \]
 Consequently taking the maximum over $i\in[n]$ in the above gives $\cL_{\max} \leq L_{\max}.$
\end{proof}

\section{Proof of Proposition \ref{pro:lmaxi}}
\begin{proof}
First note that by combining~\eqref{eq:expsmoothas} and~\eqref{eq:tempcalcneed} we have that
	\begin{eqnarray}
		\cL_{\max} &\overset{\eqref{eq:expsmoothas}}{=} &
 \max_{i\in [n]} \left\{ \sum_{C: i\in C}\frac{p_C}{p_i}L_C \right\} \nonumber \\
	& \overset{\eqref{eq:tempcalcneed}}{=} &   \max_{i\in [n]} \left\{ 	 \frac{1}{n}\sum_{C: i\in C}\frac{p_C}{p_i}\lambda_{\max}\left(\sum_{j\in C}\frac{\mM_j}{p_j}\right) \right\}. \label{eq:otherLimaxbnd}
	\end{eqnarray}

	(i) By straight forward calculation from (\ref{eq:otherLimaxbnd}) and using that each set $C$ is a singleton. \\

	(ii) For every partition sampling we have that $p_i = p_C$ if $i\in C$, hence 
	\begin{eqnarray*}
		\cL_{\max} &\overset{\eqref{eq:otherLimaxbnd}}{=}&		
		\max_{i\in [n]} \left\{ 	 \frac{1}{n}\sum_{C: i\in C}\frac{p_i}{p_i}\lambda_{\max}\left(\sum_{j\in C}\frac{\mM_j}{p_C}\right) \right\} \\
		&\overset{\eqref{eq:tempcalcneed}}{=}& \frac{1}{n}\max_{i\in [n]} \left\{ \sum_{C: i\in C} \frac{1}{p_C}\lambda_{\max} (\sum_{j\in C}\mM_j) \right\} \\
		&=& \frac{1}{n} \max_{C\in {\cal G}} \left\{ \frac{1}{p_C} \lambda_{\max}(\sum_{j\in C}\mM_j) \right\}. 
	\end{eqnarray*}
	
\end{proof}

\section{Proof of Proposition \ref{pro:cL}} 
\label{sec:A:pro:cL}
\begin{proof}

 First, since $f_i$ is $L_i$-smooth  with $L_i = \lambda_{\max}({\bf M}_i)$ and convex,  it follows from equation (2.1.7) in Theorem 2.1.5 in~\cite{nesterov2013introductory}  that 
\begin{equation}\label{eq:procL1}
\|\nabla f_i(x) - \nabla f_i(y)\|^2  \leq 2L_i(f_i(x) - f_i(y) - \langle \nabla f_i(y), x-y\rangle ). 
\end{equation}

Since $f$ is $L$-smooth, we have 
\begin{equation}\label{eq:procL2}
\|\nabla f(x) - \nabla f(y)\|^2  \leq 2L(f(x) - f(y) - \langle \nabla f(y), x-y\rangle ). 
\end{equation}

Noticing that 
\begin{eqnarray*}
\|\nabla f_v(x) - \nabla f_v(y)\|^2 &=& \frac{1}{n^2} \left \|\sum_{i\in S}\frac{1}{p_i}(\nabla f_i(x) - \nabla f_i(y)) \right \|^2 \\
&=& \sum_{i,j\in S} \left\langle \frac{1}{np_i}(\nabla f_i(x) - \nabla f_i(y)), \frac{1}{np_j}(\nabla f_j(x) - \nabla f_j(y)) \right\rangle,
\end{eqnarray*}
we have 
\begin{eqnarray*}
\mathbb{E}[\|\nabla f_v(x) - \nabla f_v(y)\|^2] &=& \sum_C p_C  \sum_{i,j\in C} \left\langle \frac{1}{np_i}(\nabla f_i(x) - \nabla f_i(y)), \frac{1}{np_j}(\nabla f_j(x) - \nabla f_j(y)) \right\rangle \\ 
&=& \sum_{i, j=1}^n \sum_{C: i,j\in C }p_C  \left\langle \frac{1}{np_i}(\nabla f_i(x) - \nabla f_i(y)), \frac{1}{np_j}(\nabla f_j(x) - \nabla f_j(y)) \right\rangle \\ 
&=& \sum_{i, j=1}^n \frac{\mP_{ij}}{p_ip_j} \left\langle \frac{1}{n}(\nabla f_i(x) - \nabla f_i(y)), \frac{1}{n}(\nabla f_j(x) - \nabla f_j(y)) \right\rangle. 
\end{eqnarray*}

Now consider the case where $\mP_{ij}/(p_ip_j) = c_2$ for $i \neq j.$ Recalling that $\mP_{ii} =p_i$ we have from the above that

\begin{eqnarray*}
\mathbb{E}[\|\nabla f_v(x) - \nabla f_v(y)\|^2] &=& 
\sum_{i \neq j} c_2 \left\langle \frac{1}{n}(\nabla f_i(x) - \nabla f_i(y)), \frac{1}{n}(\nabla f_j(x) - \nabla f_j(y)) \right\rangle + \sum_{i=1}^n\frac{1}{n^2} \frac{1}{p_i} \norm{\nabla f_i(x) - \nabla f_i(y))}_2^2 \\
&= &  \sum_{i,j=1}^n c_2 \left\langle \frac{1}{n}(\nabla f_i(x) - \nabla f_i(y)), \frac{1}{n}(\nabla f_j(x) - \nabla f_j(y)) \right\rangle \\
& &+ \sum_{i=1}^n\frac{1}{n^2} \frac{1}{p_i}\left(1 -p_ic_2 \right) \norm{\nabla f_i(x) - \nabla f_i(y))}_2^2\\
& \overset{\eqref{eq:procL1}}{\leq} &  c_2 \norm{\nabla f(x) - \nabla f(y)}_2^2 \\
& &+ 2 \sum_{i=1}^n\frac{1}{n^2} \frac{L_i}{p_i}\left(1 -p_ic_2 \right) (f_i(x) - f_i(y) - \langle \nabla f_i(y), x-y\rangle ) \\
& \overset{\eqref{eq:procL2}}{\leq} & 2\left(c_2 L +\max_{i=1,\ldots, n}\frac{L_i}{np_i}\left(1 -p_ic_2 \right)  \right) (f(x) - f(y) - \langle \nabla f(y), x-y\rangle ).
\end{eqnarray*}
Substituting $y =x^*$ and comparing the above to the definition of expected smoothness~\eqref{eq:expsmooth} we have that
\begin{equation} \label{eq:CLinterpolc2}
\cL \quad \leq \quad  c_2 L +\max_{i=1,\ldots, n}\frac{L_i}{np_i}\left(1 -p_ic_2 \right).
\end{equation}

(i) For independent sampling, we have that ${\bf P}_{ij} = p_ip_j$ for $i \neq j$, consequently $c_2 =1.$ Thus~\eqref{eq:CLinterpolc2} gives~\eqref{eq:cLis}.

(ii) For $\tau$-nice sampling, we have that $\mP_{ij} = \frac{\tau(\tau-1)}{n(n-1)}$ for $j \neq i$ and $\mP_{ii} = p_i = \frac{\tau}{n}$, hence $c_2 = \frac{n(\tau-1)}{\tau(n-1)}$ and \eqref{eq:CLinterpolc2} gives~\eqref{eq:cLns}.
\end{proof}

\section{Proof of Theorem~\ref{lem:sigma}}

\begin{proof}
	\begin{eqnarray*}
		\sigma^2 = \mathbb{E}[\|\nabla f_v(x^*) \|^2] &=& \mathbb{E}\left[ \left\| \frac{1}{n}\sum_{i=1}^n\nabla f_i(x^*)v_i \right\|^2\right] = \frac{1}{n^2}\mathbb{E}\left[\left\| \sum_{i=1}^n \nabla f_i(x^*)v_i \right\|^2\right]= \frac{1}{n^2} \mathbb{E} \left[\left\| \sum_{i\in {S}} \frac{1}{p_i}h_i \right\|^2\right] \\
		&=& \frac{1}{n^2} \mathbb{E} \left[ \left\| \sum_{i=1}^n 1_{i \in { S}}\frac{1}{p_i}h_i\right\|^2 \right] = \frac{1}{n^2} \mathbb{E} \left[\sum_{i=1}^n \sum_{j=1}^n 1_{i \in { S}} 1_{j \in { S}} \langle \frac{1}{p_i}h_i, \frac{1}{p_j}h_j \rangle \right]\\
		&=& \frac{1}{n^2}\sum_{i,j}\frac{\mP_{ij}}{p_ip_j} \langle h_i, h_j \rangle .
	\end{eqnarray*}
\end{proof}

\section{Proof of Proposition \ref{pro:sigma}}

\begin{proof}
	(i) By straight calculation from (\ref{eq:sigma}). \\
	
	(ii) For independent sampling $S$, $\mP_{ij} = p_ip_j$ for $i\neq j$, hence, 
	\begin{eqnarray*}
		\sigma^2 &=& \frac{1}{n^2} \sum_{i,j\in [n]} \frac{\mP_{ij}}{p_ip_j} \langle h_i, h_j \rangle = \frac{1}{n^2} \sum_{i,j\in [n]} \langle h_i, h_j \rangle + \frac{1}{n^2} \sum_{i\in [n]} \left(\frac{1}{p_i} -1\right)\|h_i\|^2 \\
		&=& \frac{1}{n^2} \|\nabla f(x^*) \|^2 +  \frac{1}{n^2} \sum_{i\in [n]} \left(\frac{1}{p_i} -1\right)\|h_i\|^2  = \frac{1}{n^2} \sum_{i\in [n]} \left(\frac{1}{p_i} -1\right)\|h_i\|^2. 
	\end{eqnarray*}
	
	(iii) For $\tau$-nice sampling $S$, if $\tau =1$, it is obvious. If $\tau \geq 1$,
	 then $\mP_{ij} = \frac{C_{n-2}^{\tau-2}}{C_n^\tau}$ for $i\neq j$, and $p_i = \frac{\tau}{n}$ for all $i$. Hence, 
	\begin{eqnarray*}
		\sigma^2 &=& \frac{1}{n^2} \sum_{i,j\in [n]} \frac{\mP_{ij}}{p_ip_j} \langle h_i, h_j \rangle \\
		&=& \frac{1}{n^2} \sum_{i\neq j} \frac{\tau(\tau-1)}{n(n-1)}\cdot \frac{n^2}{\tau^2} \langle h_i, h_j \rangle + \frac{1}{n^2} \sum_{i\in [n]} \frac{n}{\tau} \|h_i\|^2 \\
		&=& \frac{1}{n\tau} \left( \sum_{i\neq j} \frac{\tau-1}{n-1}\langle h_i, h_j \rangle +  \sum_{i\in [n]} \|h_i\|^2 \right)\\ 
		&=& \frac{1}{n\tau} \left( \sum_{i,j\in [n]} \frac{\tau-1}{n-1}\langle h_i, h_j \rangle +  \sum_{i\in [n]} \frac{n-\tau}{n-1}\|h_i\|^2 \right)\\  
		&=& \frac{1}{n\tau} \cdot \frac{n-\tau}{n-1} \sum_{i\in [n]} \|h_i\|^2. 
	\end{eqnarray*}

(iv) For partition sampling, $\mP_{ij} = p_C$ if $i, j\in C$, and $\mP_{ij} = 0$ otherwise. Hence, 
\[
	\sigma^2 = \frac{1}{n^2} \sum_{i,j\in [n]} \frac{\mP_{ij}}{p_ip_j} \langle h_i, h_j \rangle = \frac{1}{n^2} \sum_{C\in {\cal G}} \sum_{i, j\in C} \frac{1}{p_C} \langle h_i, h_j \rangle = \frac{1}{n^2}\sum_{C\in {\cal G}} \frac{1}{p_C} \|\sum_{i\in C}h_i\|^2.
\]

\end{proof}

\section{Importance sampling}
\label{sec:important}
\subsection{Single element sampling}
\label{eq:importancesamp}
From (\ref{eq:lmaxi1}) it is easy to see that the probabilities that minimize $\cL_{\max}$ are 
$
p_i^{\cal L} = \left. L_i \right/ \sum_{j\in [n]}L_j,
$
for all $i$, and consequently  ${\cal L}_{\max} = \overline{L}$. On the other hand the probabilities that minimize~\eqref{eq:sigma1} are given by
$
p_i^{\sigma^2} = \left. \|h_i\| \right/ \sum_{j\in [n]}\|h_j\|,
$
for all $i$, with $\sigma^2 =  (\sum_{i\in [n]}\|h_i\| /n)^2 \eqdef \sigma_{opt}^2$. 

\paragraph{Importance sampling.} From $p_i^{\cal L}$ and $p_i^{\sigma^2}$, we construct interpolated probabilities $p_i$ as follows: 
\begin{equation} \label{eq:optimal_prob_sing}
p_i = p_i(\alpha) = \alpha p_i^{\cal L} + (1-\alpha)p_i^{\sigma^2},
\end{equation}
where $\alpha \in (0, 1)$. Then $0< p_i< 1$ and from (\ref{eq:lmaxi1}) we have 
$$
\cL_{\max} \leq \frac{1}{\alpha} \cdot \frac{1}{n} \max_{i\in [n]} \frac{L_i}{p_i^{\cal L}(\tau)} = \frac{1}{\alpha} \overline{L}.
$$
Similarly, from~(\ref{eq:sigma1}) we have that $\sigma^2 \leq \frac{1}{1-\alpha}\sigma^2_{opt}$. Now by letting $p_i = p_i(\alpha)$, from (\ref{eq:itercomplexlinis}) in Theorem \ref{theo:strcnvlin}, we get an upper bound of the right hand side of (\ref{eq:itercomplexlin}): 
\begin{equation}\label{eq:alpha1iterbound}
\max\left\{ \frac{2\overline{L}}{\alpha \mu}, \frac{4\sigma^2_{opt}}{(1-\alpha)\epsilon\mu^2} \right\}.
\end{equation}
By minimizing this bound in $\alpha$ we can get 
\begin{equation}\label{eq:alpha1}
\alpha = \frac{\overline{L}}{2\sigma^2_{opt}/\epsilon\mu + \overline{L}},
\end{equation}
and then the upper bound~\eqref{eq:alpha1iterbound} becomes 
\begin{equation}\label{eq:alpha1bound}
\frac{4\sigma^2_{opt}}{\epsilon \mu^2} + \frac{2  \overline{L} }{ \mu} \leq 2\max \left\{ \frac{2\overline{L}}{ \mu}, \frac{4\sigma^2_{opt}}{\epsilon\mu^2} \right\},
\end{equation}
where the right hand side comes by setting $\alpha = 1/2$. Notice that the minimum of the iteration complexity in (\ref{eq:itercomplexlin}) is not less than $\max \left\{ \frac{2\overline{L}}{ \mu}, \frac{4\sigma^2_{opt}}{\epsilon\mu^2} \right\}$. Hence, the iteration complexity of this importance sampling(left hand side of (\ref{eq:alpha1bound})) is at most two times larger than the minimum of the iteration complexity in (\ref{eq:itercomplexlin}) over $p_i$. 

\subsection{Independent sampling}
\label{sec:indepensamp}

For the independent sampling $S$, in this section we will use the following upper bound on $\cL$ given by
	\begin{equation}\label{eq:lmaxiis}
	 \cL_{\max} \leq   \sum_{i} \frac{L_i}{n} + \max_{i\in [n]} \frac{1-p_i}{p_i} \frac{L_i}{n},
	\end{equation}
which follows immediatly from~\eqref{eq:cLis} by using that $L \leq \frac{1}{n}\sum_{i=1}^n L_i \eqdef \overline{L}.$
\paragraph{Calculating $p_i^{\cL}(\tau)$.} Minimizing the upper bound of $\cL_{\max}$ in~\eqref{eq:lmaxiis} boils down to minimizing $\max_{i\in [n]} (\frac{1}{p_i}-1)L_i$, which is not easy generally. Instead, as a proxy we obtain the probabilities $p_i$ by solving  
\begin{equation}\label{eq:lmaxiis22}
\begin{array}{rcl}
\min & & \max_{i\in [n]}\frac{L_i}{p_i} \\
{\rm s.t.\ }& & \sum_{i\in [n]}p_i = \tau, \  0< p_i \leq 1, \forall i. 
\end{array}
\end{equation}
Let $q_i = \frac{L_i}{\sum_{j\in [n]}L_j}\cdot \tau$ for all $i$, and $T = \{ i | q_i >1 \}$. If $T = \emptyset$, it is easy to see $p_i = p_i^{\cal L}(\tau)= q_i$ solves (\ref{eq:lmaxiis22}). Otherwise, in order to solve (\ref{eq:lmaxiis22}), we can choose $p_i = p_i^{\cal L}(\tau) =1$ for $i\in T$, and $q_i\leq p_i = p_i^{\cal L}(\tau) \leq 1$ for $i\notin T$ such that $\sum_{i\in [n]} p_i^{\cal L}(\tau) = \tau$. By letting $p_i = p_i^{\cal L}(\tau)$, we have that~\eqref{eq:lmaxiis22} becomes
\begin{equation} \label{eq:cLtempappis}
	\cL_{\max} \leq \frac{1}{n} \left( \left(1+ \frac{1}{\tau}\right)\sum_{j\in [n]} L_j - \min_{j\in [n]} L_j \right) \leq  \left(1+\frac{1}{\tau}\right) \overline{L}.
\end{equation}
\paragraph{Calculating $p_i^{\sigma^2}(\tau)$.}For $\sigma^2$, from (\ref{eq:sigmais}), we need to solve 
\begin{equation}\label{eq:sigmais22}
\begin{array}{rcl}
\min & & \sum_{i\in [n]} \frac{\|h_i\|^2}{p_i} \\
{\rm s.t.\ }& & \sum_{i\in [n]}p_i = \tau, \  0< p_i \leq 1, \forall i. 
\end{array}
\end{equation}
Let $q_i = \frac{\|h_i\|}{\sum_{j\in [n]} \|h_j\|}\cdot \tau$ for all $i$, and let $T = \{ i | q_i >1  \}$. If $T = \emptyset$, it is easy to see that $p_i = p_i^{\sigma^2}(\tau) = q_i$ solve (\ref{eq:sigmais22}). Otherwise, it is a little complicated to find the optimal solution. For simplicity, if $T \neq \emptyset$, we choose $p_i = p_i^{\sigma^2}(\tau) = 1$ for $i\in T$, and $q_i \leq p_i = p_i^{\sigma^2}(\tau) \leq 1$ for $i\notin T$ such that $\sum_{i\in [n]} p_i^{\sigma^2}(\tau)= \tau$. By letting $p_i = p_i^{\sigma^2}(\tau)$, from (\ref{eq:sigmais}), we have 
\begin{eqnarray*}
	\sigma^2 &\leq& \frac{1}{n^2} \sum_{i\notin T} \left( \frac{\|h_i\| \sum_{j\in [n]}\|h_j\|}{\tau} - \|h_i\|^2 \right) \\ 
	&\leq& \frac{1}{\tau}\left(\frac{\sum_{i\in [n]}\|h_i\|}{n}\right)^2 \eqdef \sigma^2_{opt}(\tau). 
\end{eqnarray*}

\paragraph{Importance sampling.} Since by~\eqref{eq:cLtempappis} we have that $\cL_{\max} \leq \left(1+\frac{1}{\tau}\right) \overline{L}$ and $\sigma = \sigma^2_{opt}(\tau)$ are obtained by using the upper bounds in (\ref{eq:lmaxiis}) and (\ref{eq:sigmais}), and the upper bounds are nonincreasing as $p_i$ increases,  we get the following property. 
\begin{proposition}\label{pro:landsigma}
	If $p_i \geq p_i^{\cal L}(\tau)$ for all $i$, then $\cL_{\max} \leq  (1+\frac{1}{\tau}) \overline{L}$, and if $p_i \geq p_i^{\sigma^2}(\tau)$, then $\sigma^2 \leq \sigma^2_{opt}(\tau)$.
\end{proposition}

From Proposition \ref{pro:landsigma}, we can get the following result. 

\begin{proposition}\label{pro:pialpha}
For $0< \alpha <1$, let $p_i(\alpha)$ satisfy 
\begin{equation}\label{pialphais}
\left\{
\begin{array}{l}
1\geq p_i(\alpha) \geq \min\{ 1, p_i^{\cal L}(\alpha \tau) + p_i^{\sigma^2}((1-\alpha)\tau) \}, \quad \forall i, \\
\sum_{i\in [n]}p_i(\alpha) = \tau.
\end{array}\right.
\end{equation}
If $p_i = p_i(\alpha)$ where $p_i(\alpha)$ satisfies (\ref{pialphais}), then we have 
$$
\cL_{\max} \leq \left(1+\frac{1}{\alpha \tau}\right)\overline{L},
$$
and 
$$
\sigma^2 \leq \sigma^2_{opt}((1-\alpha)\tau) = \frac{1}{(1-\alpha)\tau}(\frac{\sum_{i\in [n]}\|h_i\|}{n})^2.
$$
\end{proposition}

\begin{proof}
First , we claim that $p_i(\alpha)$ can be constructed to satisfy (\ref{pialphais}). Since $0< p_i^{\cL}(\alpha) \leq 1$ and $0< p_i^{\sigma^2}((1-\alpha)\tau) \leq 1$, we know 
$$
0< \min\{ 1, p_i^{\cal L}(\alpha \tau) + p_i^{\sigma^2}((1-\alpha)\tau) \} \leq 1,
$$
for all $i$. Hence, we can first construct ${\tilde q}_i$ such that 
$$
1\geq {\tilde q}_i \geq \min\{ 1, p_i^{\cal L}(\alpha \tau) + p_i^{\sigma^2}((1-\alpha)\tau) \},
$$
for all $i$. Furthermore, since $\sum_{i\in [n]} p_i^{\cal L}(\alpha \tau) = \alpha \tau$ and $\sum_{i\in [n]} p_i^{\sigma^2}((1-\alpha)\tau) = (1-\alpha)\tau$, we know $\sum_{i\in [n]} {\tilde q}_i \leq \tau$. At last, we increase some ${\tilde q}_i$ which is less than one to make the sum equal to $\tau$, and hence, by letting $p_i(\alpha) = {\tilde q}_i$, $p_i(\alpha)$ satisfies (\ref{pialphais}). 

From (\ref{pialphais}), we have $p_i = p_i(\alpha) \geq p_i^{\cL}(\alpha\tau)$. Then by Proposition \ref{pro:landsigma}, we have 
$$
\cL_{\max} \leq \left(1+\frac{1}{\alpha \tau}\right)\overline{L}.
$$
We also have $p_i(\alpha) \geq p_i^{\sigma^2}((1-\alpha)\tau)$, hence, by Proposition \ref{pro:landsigma}, we get 
$$
\sigma^2 \leq \sigma^2_{opt}((1-\alpha)\tau) = \frac{1}{(1-\alpha)\tau}\left(\frac{\sum_{i\in [n]}\|h_i\|}{n}\right)^2.
$$
\end{proof}

From (\ref{eq:itercomplexlin}) in Theorem \ref{theo:strcnvlin}, by letting $p_i = p_i(\alpha)$ in Proposition \ref{pro:pialpha}, we get an upper bound of the right hand side of (\ref{eq:itercomplexlin}): 
$$
\max\left\{ \frac{2 (1+\frac{1}{\alpha \tau}) \overline{L}}{ \mu}, \frac{4\sigma^2_{opt}((1-\alpha)\tau)}{\epsilon\mu^2} \right\}.
$$
By minimizing this upper bound, we get 
\begin{equation} \label{eq:alpha_star}
\alpha = \frac{\tau -a-1 + \sqrt{4\tau + (\tau -a-1)^2}}{2\tau},
\end{equation}
and the upper bound becomes  
$$
\frac{2(1+\frac{1}{\alpha\tau})}{\mu}\cdot \overline{L} 
$$
where $a = 2(\frac{\sum_{i\in [n]}\|h_i\|}{n})^2/(\epsilon \mu \overline{L} )$. 
So suboptimal probabilities 
\begin{equation} \label{eq: optimal_prob_independ}
p_i = \min\{ 1, p_i^{\cal L}(\alpha \tau) + p_i^{\sigma^2}((1-\alpha)\tau) \},
\end{equation}
where $\alpha$ is given in Equation~(\ref{eq:alpha_star}).

\paragraph{Partially biased sampling.} In practice, we do not know $\|h_i\|$ generally. But we can use $p_i^{\cL}(\tau)$ and the uniform probability $\frac{\tau}{n}$ to construct a new probability just as that in Proposition \ref{pro:pialpha}. More specific, we have the following result. 

\begin{proposition}\label{pro:pialphapbs}
	Let $p_i$ satisfy 
	\begin{equation}\label{pialphapbs}
	\left\{
	\begin{array}{l}
	1\geq p_i \geq\min\{1, p_i^{\cal L}(\frac{\tau}{2}) +  \frac{1}{2}\cdot \frac{\tau}{n}  \}, \quad \forall i, \\
	\sum_{i\in [n]}p_i = \tau.
	\end{array}\right.
	\end{equation}
	Then we have 
	$$
	\cL_{\max} \leq  \left(1+\frac{2}{ \tau}\right) \overline{L},
	$$
	and 
	$$
	\sigma^2 \leq \left(\frac{2}{\tau} - \frac{1}{n}\right)\cdot \frac{1}{n}\sum_{i\in [n]}\|h_i\|^2. 
	$$
\end{proposition}

\begin{proof}
The proof for $\cL_{\max}$ is the same as Proposition \ref{pro:pialpha}. For $\sigma^2$, from  (\ref{eq:sigmais}), since $p_i \geq \tau/2n$, we have 
\[
\sigma^2 = \frac{1}{n^2} \sum_{i\in [n]} \left(\frac{1}{p_i}-1\right)\|h_i\|^2 
\leq \frac{1}{n^2} \sum_{i\in [n]} \left(\frac{2n}{\tau}-1\right)\|h_i\|^2 
= \left(\frac{2}{\tau} - \frac{1}{n}\right)\cdot \frac{1}{n}\sum_{i\in [n]}\|h_i\|^2.
\]

\end{proof}

This sampling is very nice in the sense that it can maintain ${\cal L}_{\max}$ at least close to $\overline{L}$, and meanwhile, can acheive nearly linear speedup in $\sigma^2$ by increasing $\tau$. We can compare the upper bounds of ${\cal L}_{\max}$ and $\sigma^2$ for this sampling, $\tau$-nice sampling, and $\tau$-uniform independent sampling when $1< \tau = {\cal O}(1)$ in the following table.   

\begin{table}[t]
	\caption{Comparison of the upper bounds of ${\cal L}_{\max}$ and $\sigma^2$ for $\tau$-nice sampling, $\tau$-partially biased independent sampling, and $\tau$-uniform independent sampling.}
	\label{table1}
	\vskip 0.15in
	\begin{center}
			\begin{sc}
				\begin{tabular}{lcccr}
					\toprule
					& ${\cal L}_{\max}$ & $\sigma^2$  \\
					\midrule
					$\tau$-nice sampling   & $\frac{n}{\tau}\cdot \frac{\tau-1}{n-1}{\bar L} + \frac{1}{\tau}(1-\frac{\tau-1}{n-1})L_{\max}$ & $\frac{1}{\tau}\cdot \frac{n-\tau}{n-1}{\bar h}$ \\
					 &  & \\
					$\tau$-uniform is & ${\bar L} +(\frac{1}{\tau} - \frac{1}{n})L_{\max}$ & $(\frac{1}{\tau} - \frac{1}{n}){\bar h}$\\
					&  & \\
					$\tau$-pba-is    & $(1+ \frac{2}{\tau}){\bar L}$ & $(\frac{2}{\tau} - \frac{1}{n}){\bar h}$ \\
					\bottomrule
				\end{tabular}
			\end{sc}
	\end{center}
	\vskip -0.1in
\end{table}

From Table \ref{table1}, compared to $\tau$-nice sampling and $\tau$-uniform independent sampling, the iteration complexity of this $\tau$-partially biased independent sampling is at most two times larger, but could be about $\frac{2\tau}{n}$ smaller in some extremely case where 
$L_{\max} \approx n{\bar L}$ and $2{\cal L}/\mu$ dominates in (\ref{eq:itercomplexlin}). 

\section{Additional Experiments}
\label{addExp}

\subsection{From fixed to decreasing stepsizes: analysis of the switching time}

Here we evaluate the choice of the switching moment from a constant to a decreasing step size according to~\eqref{eq:gammakdef} from Theorem~\ref{theo:decreasingstep}. We are using synthetic data that was generated in the same way as it had been in the Section~\ref{exper} for the ridge regression problem $(n=1000, d=100)$. In particular we evaluate 4 different cases: (i) the theoretical  moment of regime switch at moment $k$ as predicted from the Theorem, (ii) early switch at $ 0.3 \times k$, (iii) late switch at $ 0.7 \times k$ and (iv) the optimal $k$ for switch, where the optimal $k$ is obtained using one-dimensional numerical minimization of~\eqref{eq:sdaiuna3} as a function of $k^*$.
\begin{figure}[H] 
 \centering
    \includegraphics[width=0.6\linewidth]{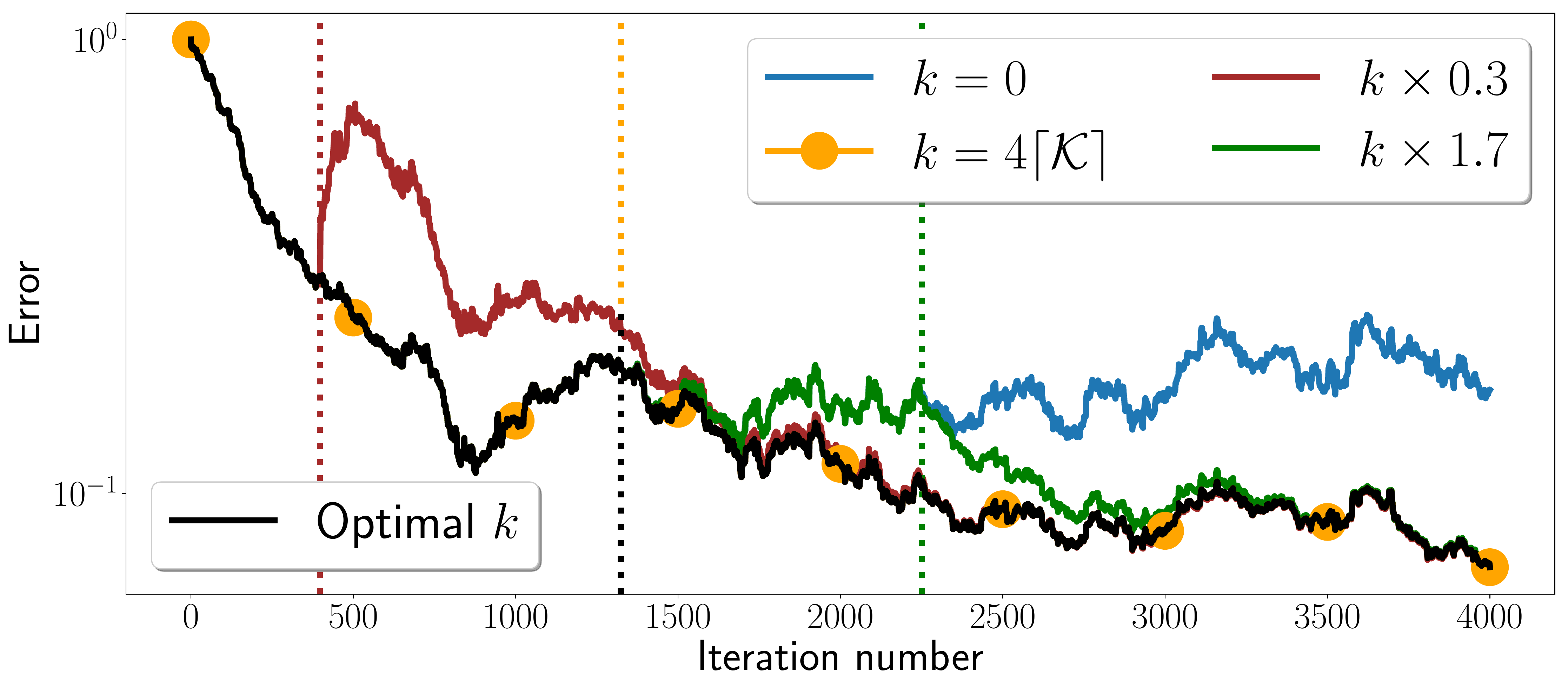}
    \includegraphics[width=0.6\linewidth]{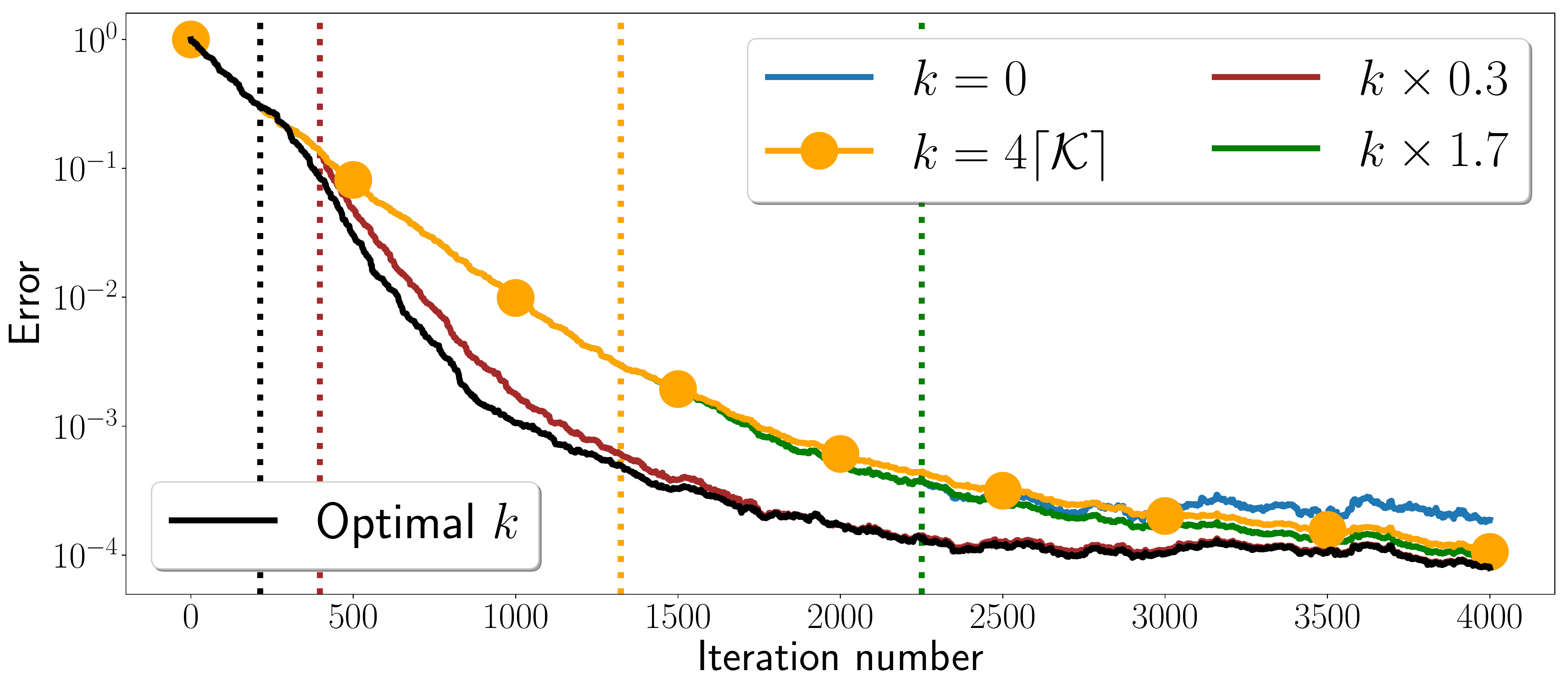}
    \caption{The first plot refers to situation when $x^0$ is close to $x^*$ (for our data $\norm{r^0}^2 = \norm{x^0 - x^*}^2 \approx 1.0$).
The second one covers the opposite case $(\norm{r^0}^2 \approx 864.6)$. Dotted verticals denote the moments of regime switch for the curves of the corresponding colour. The blue curve refers to constant step size $\frac{1}{2 \mathcal{L}}$. Notice that in the upper plot optimal and theoretical $k$ are very close}
    \label{fig:k_choice_}
\end{figure}

According to Figure~\ref{fig:k_choice_}, when $x^0$ is close to $x^*$, the moment of regime switch does not play a significant role in minimizing the number of iteration except for a very early switch, which actually also leads to almost the same situation in the long run. The case when $x^0$ is far from $x^*$ shows that preliminary one-dimensional optimization makes sense and allows to reduce the error at least during the early iterations.

\subsection{More on minibatches}

Figure~\ref{FigSamplingsXXX} reports on the same experiment as that described in Section~\ref{sec:exp2} (Figure~\ref{FigSamplings}) in the main body of the paper, but on ridge regression instead of logistic regression, and using different data sets. Our findings are similar, and corroborate the conclusions made in Section~\ref{sec:exp2}.

\begin{figure}[t] 
\centering
\begin{subfigure}{.49\textwidth}
  \centering
  \includegraphics[width=1\linewidth]{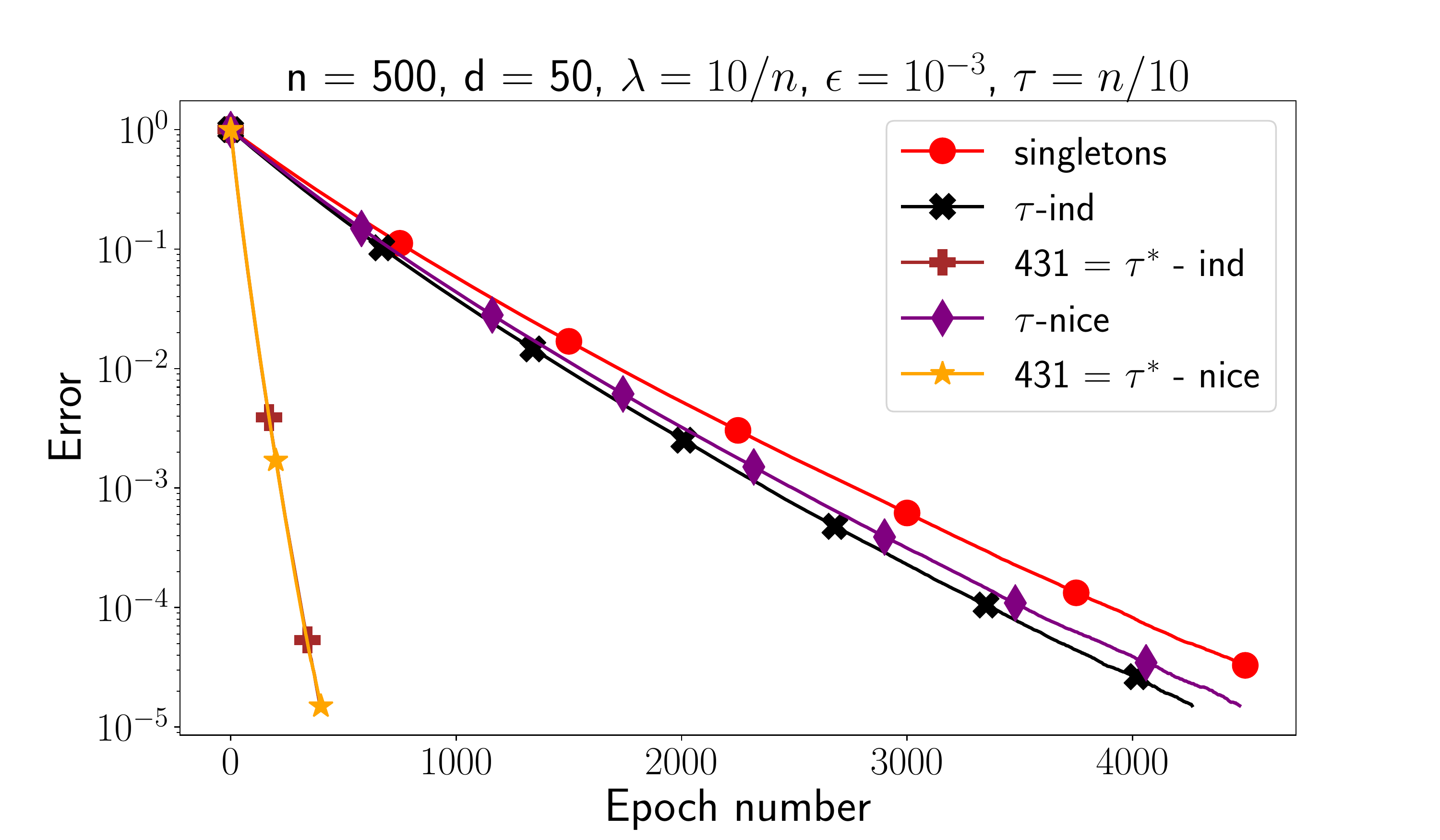}
\end{subfigure}
\begin{subfigure}{.49\textwidth}
 \centering
  \includegraphics[width=1\linewidth]{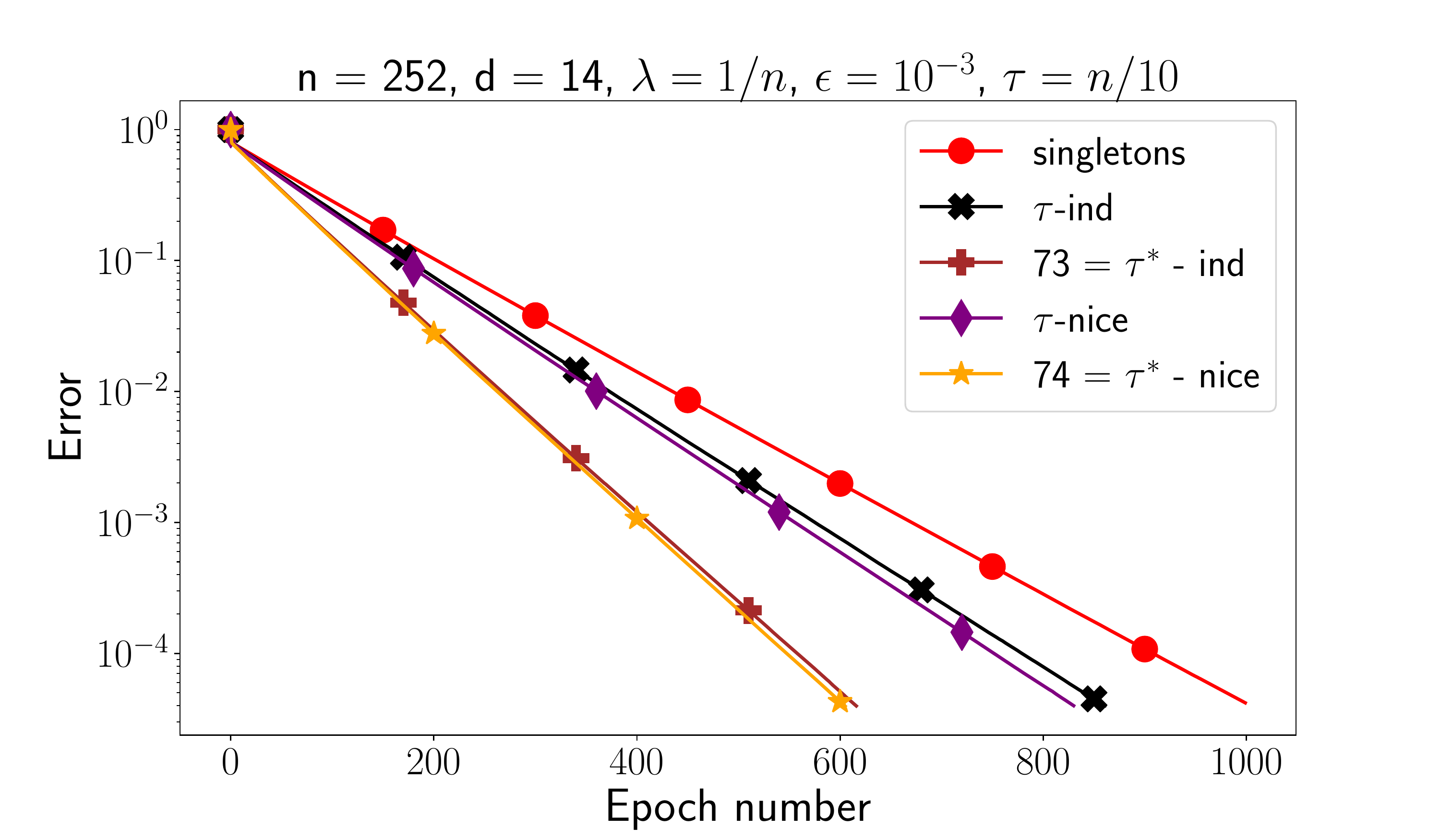}
\end{subfigure} \\
\caption{\footnotesize Performance of SGD with several minibatch strategies for  ridge regression. On the left: the real data-set bodyfat from LIBSVM. On the right: synthetic data.}
\label{FigSamplingsXXX}
\end{figure}

\subsection{Stepsize as a function of the minibatch size}

In our last experiment we calculate  the stepsize $\gamma$ as a function of the minibatch size 
 $\tau$ for $\tau$-nice sampling  using equation (\ref{eq:stepsize_tau}). Figure~\ref{fig:step_minibatch} depicts three plots, for three synthetic data sets of sizes $(n,d)\in \{(50,5),(100,10), (500,50)\}$. We consider regularized ridge regression problems with $\lambda=1/n$.  Note that  the stepsize is an increasing function of $\tau$.
 
\begin{figure}[H]
 \centering
    \includegraphics[width=0.45\linewidth]{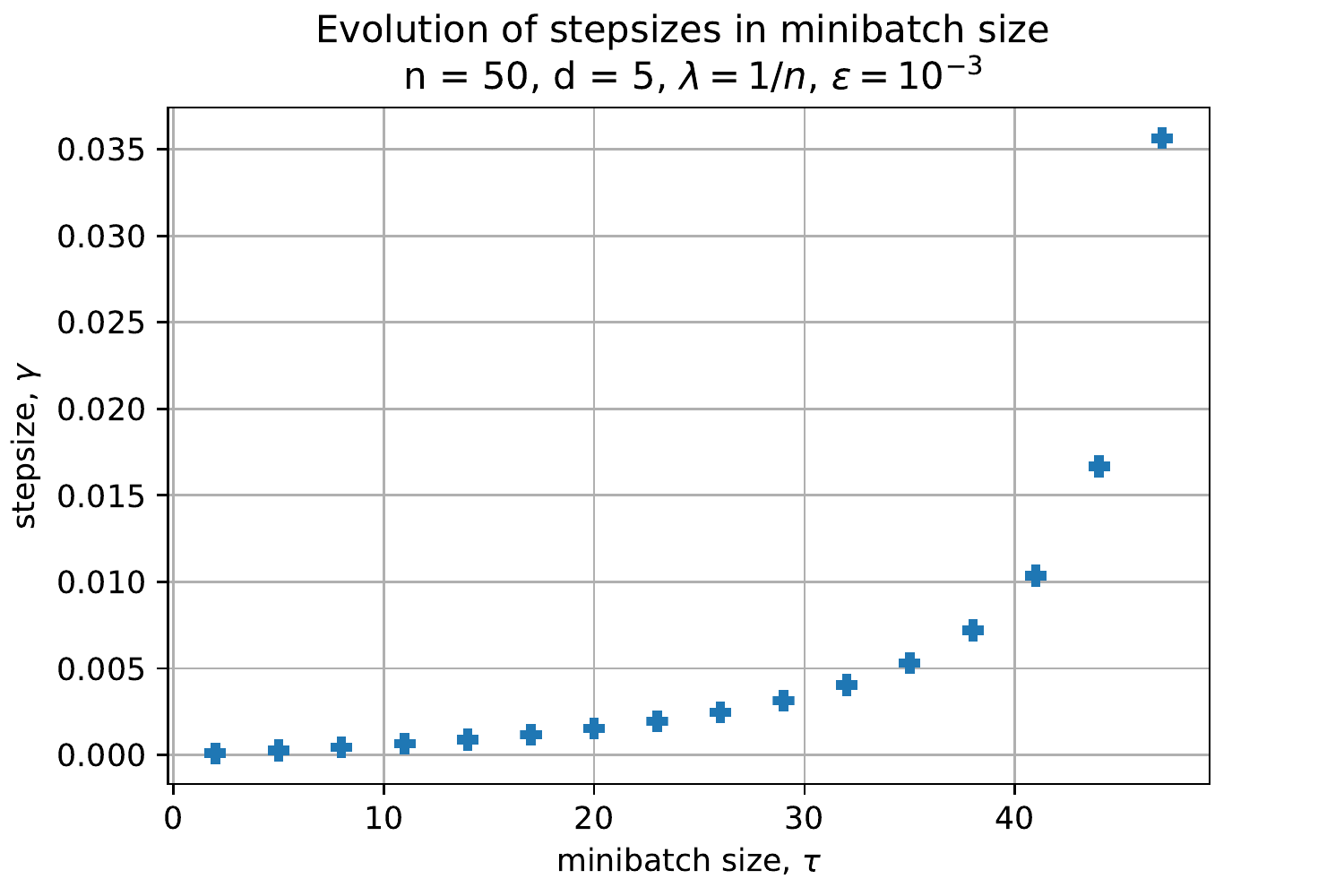}
    \includegraphics[width=0.45\linewidth]{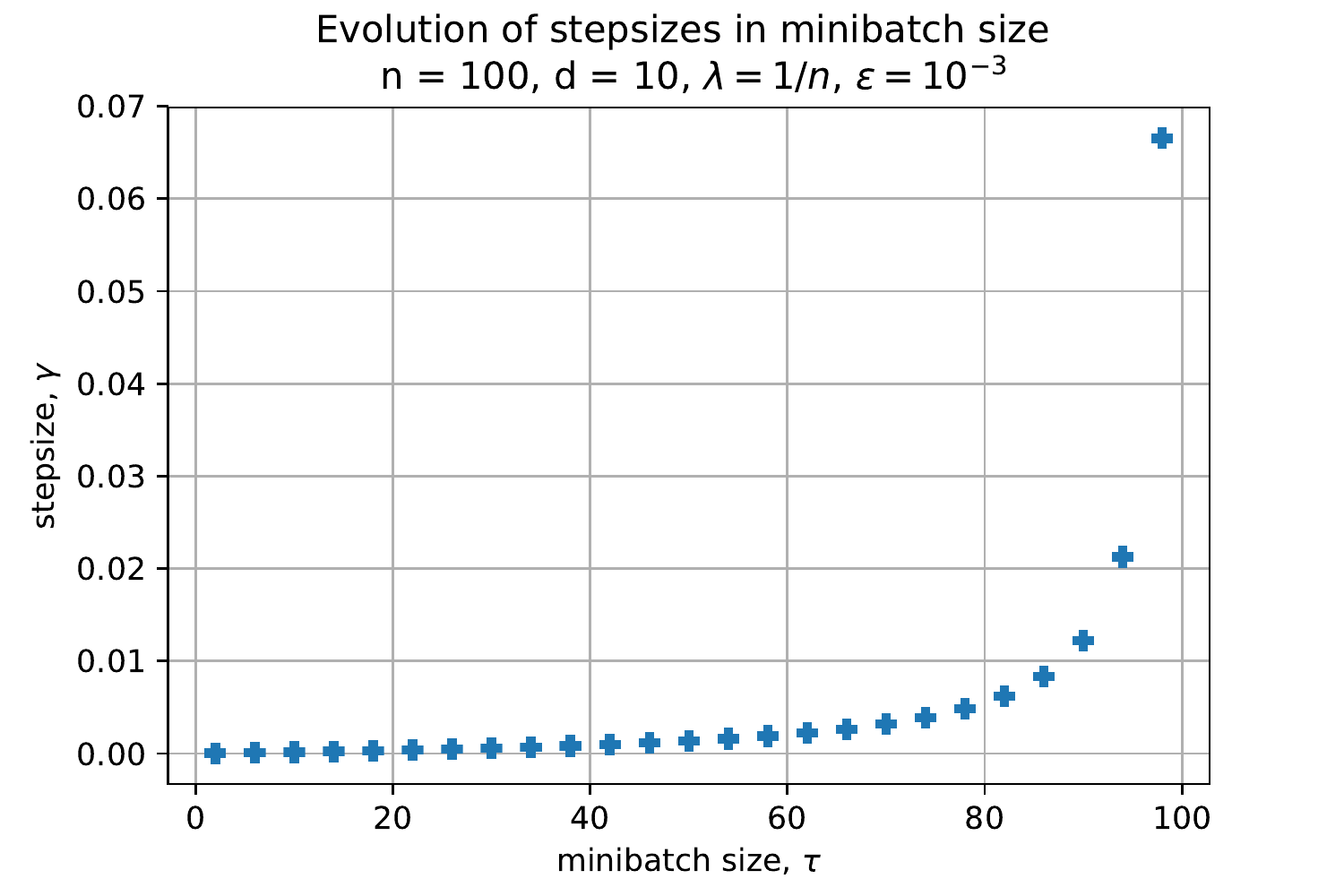}
    \includegraphics[width=0.45\linewidth]{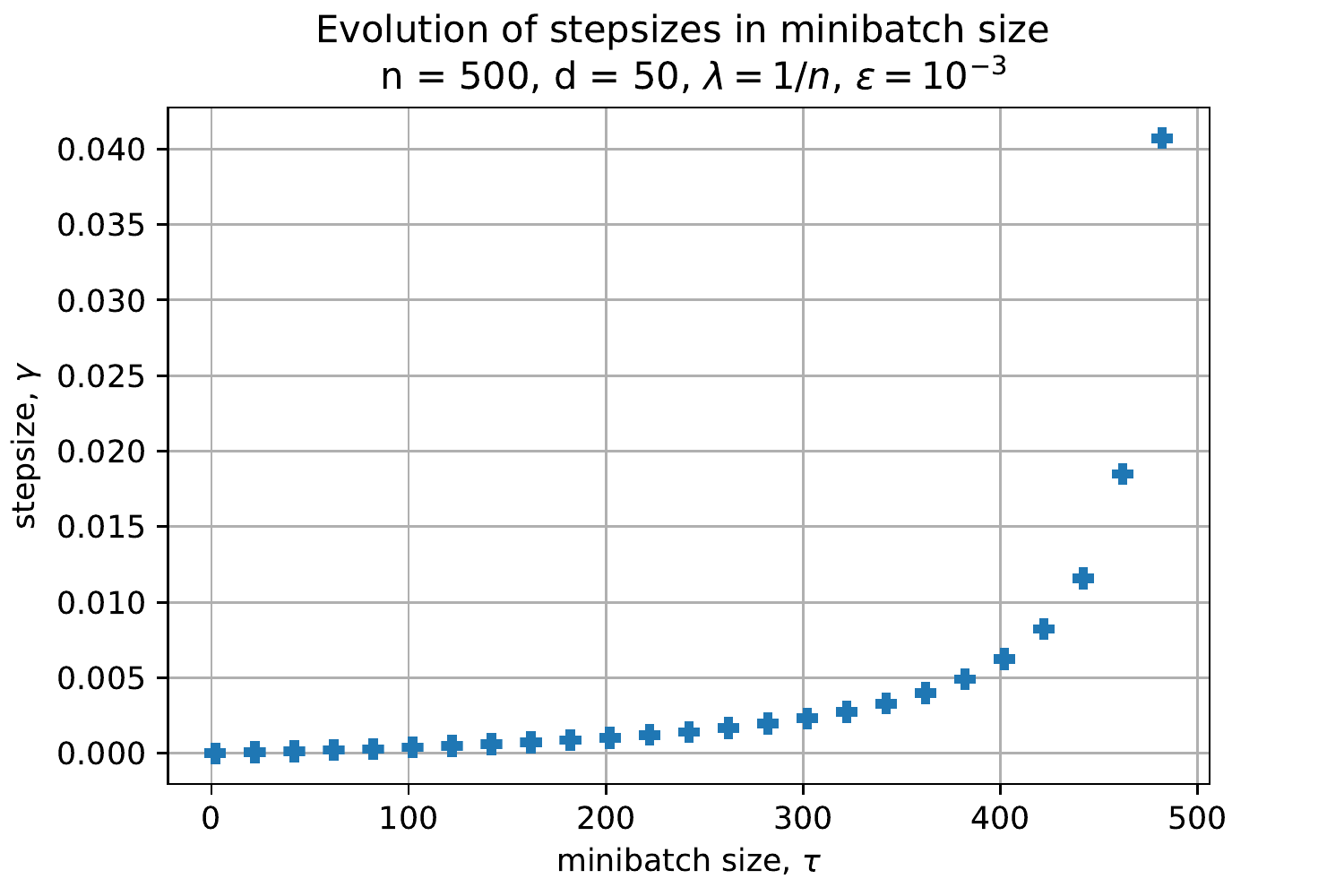}
   \caption{Evolution of stepsize with minibatch size $\tau$ for  $\tau$ nice sampling.}
    \label{fig:step_minibatch}
\end{figure}

\end{document}

%% file: SGD-AS-header.tex
\usepackage{amsmath,amsfonts,amssymb,amsthm,array}
\usepackage{amsmath}
\usepackage{algorithm}
\usepackage{algorithmic}%
\usepackage{bm}
\usepackage{color}
\usepackage{graphicx}

\newcommand{\Exp}{\mathbb{E}}

\newcommand{\E}[1]{{\mathbb{E}\left[#1\right] }}    
\newcommand{\EE}[2]{{\mathbb{E}_{#1}\left[#2\right] }} 

\newcommand{\Prob}[1]{\mathbb{P} \left[ #1\right]}
\newcommand{\R}{\mathbb{R}}

\usepackage{multirow}

\newcommand{\bA}{\mathbf{A}}

\newcommand{\eqdef}{:=}

\newcommand{\cC}{{\cal C}}
\newcommand{\cD}{{\cal D}}

\newcommand{\cL}{{\cal L}}

\newcommand{\mM}{{\bf M}}

\newcommand{\mP}{{\bf P}}

\usepackage{mdframed} 
\usepackage{thmtools}

\definecolor{shadecolor}{gray}{0.90}
\declaretheoremstyle[
headfont=\normalfont\bfseries,
notefont=\mdseries, notebraces={(}{)},
bodyfont=\normalfont,
postheadspace=0.5em,
spaceabove=1pt,
mdframed={
  skipabove=8pt,
  skipbelow=8pt,
  hidealllines=true,
  backgroundcolor={shadecolor},
  innerleftmargin=4pt,
  innerrightmargin=4pt}
]{shaded}

\declaretheorem[style=shaded,within=section]{definition}
\declaretheorem[style=shaded,sibling=definition]{theorem}
\declaretheorem[style=shaded,sibling=definition]{proposition}
\declaretheorem[style=shaded,sibling=definition]{assumption}
\declaretheorem[style=shaded,sibling=definition]{corollary}

\declaretheorem[style=shaded,sibling=definition]{lemma}
\declaretheorem[style=shaded,sibling=definition]{remark}
\declaretheorem[style=shaded,sibling=definition]{example}

\providecommand{\norm}[1]{\left\| #1\right\|}
\newcommand{\dotprod}[1]{\left< #1\right>}

\usepackage{hyperref}